\documentclass[twoside]{article}

\usepackage[preprint]{aistats2023}
\usepackage[round]{natbib}

\usepackage{mystyle}

\begin{document}

%

%

\twocolumn[

\aistatstitle{Interactive Learning with Pricing for \\ Optimal and Stable Allocations in Markets}

\aistatsauthor{ Yigit Efe Erginbas \And Soham Phade \And  Kannan Ramchandran }

\aistatsaddress{ UC Berkeley \And Salesforce Research \And UC Berkeley}]


\begin{abstract}
Large-scale online recommendation systems must facilitate the allocation of a limited number of items among competing users while learning their preferences from user feedback. As a principled way of incorporating market constraints and user incentives in the design, we consider our objectives to be two-fold: maximal social welfare with minimal instability. To maximize social welfare, our proposed framework enhances the quality of recommendations by exploring allocations that optimistically maximize the rewards. To minimize instability, a measure of users' incentives to deviate from recommended allocations, the algorithm prices the items based on a scheme derived from the Walrasian equilibria. Though it is known that these equilibria yield stable prices for markets with known user preferences, our approach accounts for the inherent uncertainty in the preferences and further ensures that the users accept their recommendations under offered prices. To the best of our knowledge, our approach is the first to integrate techniques from combinatorial bandits, optimal resource allocation, and collaborative filtering to obtain an algorithm that achieves sub-linear social welfare regret as well as sub-linear instability. Empirical studies on synthetic and real-world data also demonstrate the efficacy of our strategy compared to approaches that do not fully incorporate all these aspects.
\end{abstract}

\section{INTRODUCTION}

Online recommendation systems have become an integral part of our socioeconomic life with the rapid increases in online services that help users discover options matching their preferences. 
Despite providing efficient ways to discover information about the preferences of users, they have played a largely complementary role in searching and browsing with little consideration of the accompanying \emph{markets} within which items are offered to users at certain prices. Indeed, in many real-world scenarios, the recommendation of items
that have associated notions of limited capacities naturally give rise to a market setting where the users compete for the \emph{allocation} of these items.

Allocation constraints and underlying marketplaces are common in numerous recommendation contexts such as
point-of-interest recommendation (Yelp), 
accommodation (Airbnb, Booking.com),
e-commerce (Amazon, eBay),
ride-share (Uber, Lyft), or
job search (TaskRabbit, LinkedIn). 
As similar systems become more ubiquitous and impactful in the broader aspects of daily life, there is a huge application drive and potential for delivering recommendations that respect the requirements of the market. Therefore, it is crucial to consider market-aware recommendation systems to maximize the user experience.

\begin{figure}[t] 
\centering
\includegraphics[width= 0.4 \textwidth]{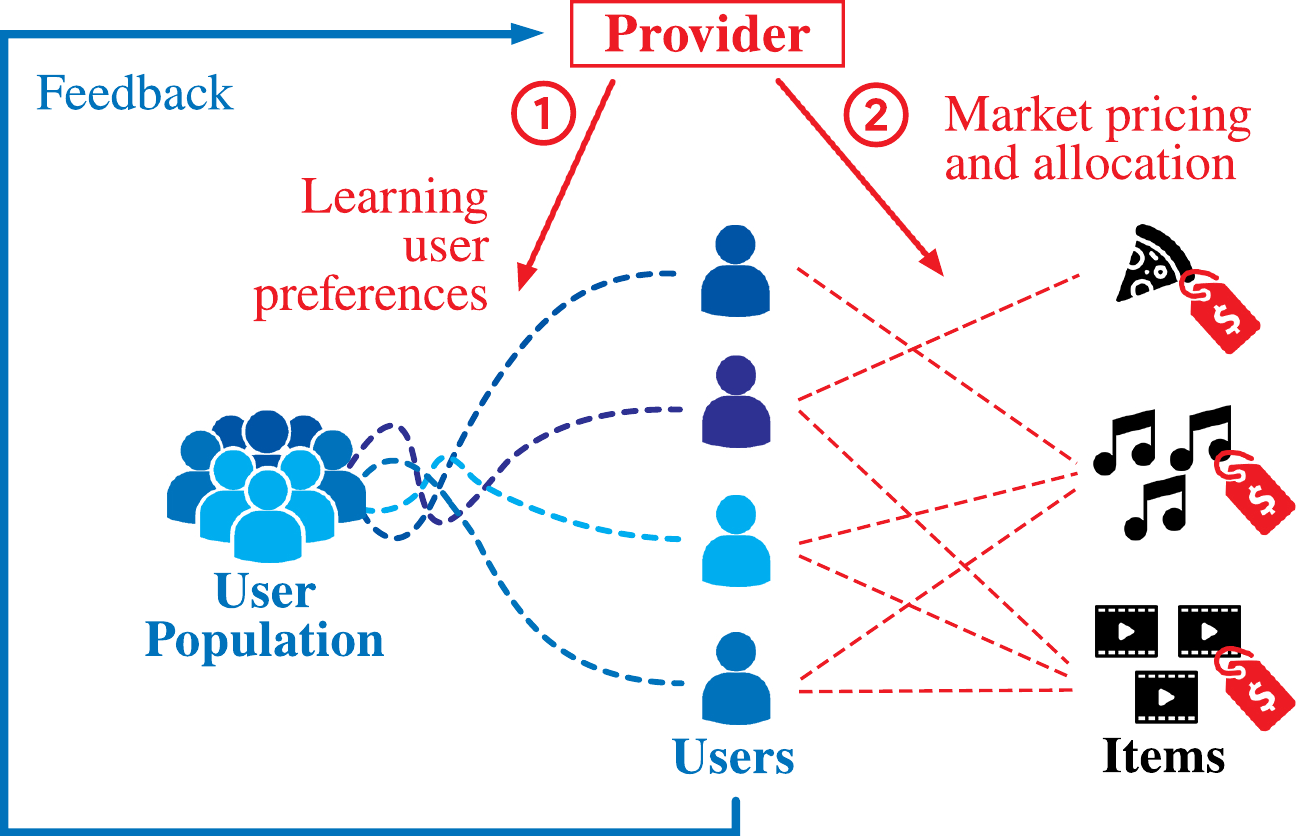}
\caption{The provider interactively learns the user preferences to achieve socially optimal and stable allocations.}
\label{system_diagram}
\vspace{-19pt}
\end{figure}

\textbf{Main Challenges: }
We model user preferences as rewards that users obtain by consuming different items, while social welfare is the aggregate reward over the entire system comprising multiple users with heterogeneous preferences. The \emph{provider}, who continually recommends items to the users, receives interactive feedback from them. The provider's goal is to maximize social welfare at all times while respecting the (time-varying) allocation constraints: indeed we consider system dynamics in terms of user demands and item capacities that change with time to be an important aspect of our problem. In the process of identifying the best allocation of items to target users, the provider encounters challenges regarding two separate aspects of the problem. The first challenge is to learn user preferences from interactive feedback, and the second is to find allocations and pricing that will result in maximal social welfare within the constraints of the market.

The first challenge of \emph{learning} arises from the fact that the provider does not have exact knowledge of the user preferences ahead of time, and hence has to learn them while continually taking action. In order to learn user preferences efficiently, previous works have established interactive recommender systems that query the users with well-chosen recommendations \citep{kawale_2015, zhao_2013}. Typically, these works consider a setting where a single user arrives at each round, and the system makes a recommendation that will match the user's preferences.
However, this approach is no longer applicable in settings having an associated market structure, as the simultaneous recommendations made to different users must also respect the constraints of the market.

This brings us to the second challenge that relates to the \emph{allocation} aspect of the problem induced by market constraints. Even if providing the users with their most preferred items would result in high rewards, such an allocation may not respect the constraints of the market in general. Therefore, to achieve better outcomes, the provider needs to take the constraints into account and only make allocations within the limits of these constraints. However, since the users care only to maximize their own gains, these allocations may be at odds with the incentives of the users. 

One common solution to this problem is to employ \emph{pricing} mechanisms. As the prices provide users with incentives to change their behavior, they can be used to induce \emph{stable} outcomes in which no user has an incentive to deviate from the recommended allocation. The celebrated theory of competitive equilibrium introduced by \cite{arrow_1951} shows that stable allocation and prices can be computed when the user preferences are \emph{fully known}. However, when the preferences are not known, it is not clear whether stability is even achievable. In this work, we address this question in detail and show that it is indeed possible to be stable by learning preferences efficiently. We accomplish this by learning the relevant user preferences efficiently and showing that instability grows only sublinearly over time so as to guarantee exact stability asymptotically.

Besides achieving stability, it is important to give users the choice of determining whether to accept or reject their recommendation based on the prices. To this end, we assume that a recommendation is accepted by a user only if the user expects to obtain more reward from this recommended item than its price. Therefore, while deciding on the prices of items, the provider also needs to be aware of user choices to achieve higher social welfare. To the best of our knowledge, we are not aware of any prior work that introduces price-based choice in online recommendation systems.

\textbf{Interactive Learning for Allocation and Pricing (ILAP): } As depicted in Figure \ref{system_diagram}, the provider has two tasks: 
(1) to learn user preferences by making recommendations that will give rise to informative responses, 
(2) to find allocations and pricing that will result in maximal social welfare with minimal instability. 
With these goals, we develop a market-aware mechanism for the provider. By recommending items, the provider helps the users to narrow down their options so that users can evaluate their preferences among a smaller set of items. In addition, being aware of the market structure, the provider determines the item prices that play the role of an intermediary for stabilizing the market.

Since there exists a trade-off between maximizing social welfare using historical data and gathering new information to improve performance in the future, the provider encounters the well-known \emph{exploration-exploitation} dilemma. In the literature of interactive recommendations, this dilemma is typically formulated as a multi-armed bandit problem where each arm corresponds to the recommendation of an item to the target user \citep{zhao_2013, barraza_2017, wang_2019}. After the provider makes recommendations to the user, the user feeds their reward information back to the provider to improve future recommendations. However, in contrast to these prior works, our setting further requires that a collection of actions taken for different users does not cause instability and respects the constraints of the market. 

Based on the standard OFU (Optimism in Face of Uncertainty) principle \citep{dani_2008, abbasi_2011}, we devise a procedure that estimates the mean reward values optimistically so as to solve the system problem of allocation and pricing. 
As is standard with OFU-based methods, our algorithm maintains a confidence set of the mean rewards for all user-item pairs. If it has less information about some user-item allocation pair, the confidence set becomes larger in the corresponding direction. Then, due to optimism, the algorithm becomes more inclined to attempt the corresponding allocation pairs to explore and collect more information. To reduce instability, it chooses the prices based on the equilibrium prices resulting from the problem with the optimistic estimation of the mean rewards. However, instead of directly offering these equilibrium prices, it reduces the prices by an amount that depends on the current confidence of the algorithm, to ensure that recommendations are likelier to be accepted.

To capture the correlation between the preferences of different users for different items, we employ latent factor (also known as matrix factorization) models that have been widely applied in recommender systems \citep{resnick_1997, sarwar_01, schafer_2007, bennett_2007, koren_2009}. We associate each user and item with an unknown feature vector and assume that expected rewards can be represented through a linear interaction of these vectors. We first consider the scenario where the features of the items are available as contexts to the provider and the provider only needs to learn the hidden information regarding the users. Then, relaxing this assumption, we consider the case where all of the feature vectors are unknown, which is equivalent to assuming that the expected reward matrix is low rank.

\textbf{Our Contributions}
\vspace{-3pt}
\begin{itemize}[nosep, labelindent= 0pt, align= left, labelsep=0.4em, leftmargin=*]
    \item We pose the problem of making recommendations that will facilitate socially optimal and stable allocation of items (with time-varying capacities) to users (with time-varying demands). 
    \item We model the structure of the preferences using linear latent factors. Considering settings with and without contextual information, we propose algorithms that achieve sublinear social welfare regret and instability bounds.
    \item For the setting where features of the items are given as context, we obtain an algorithm that achieves \smash{$\widetilde{\mathcal{O}} ( (NR)^{1/4} (nMT)^{3/4})$} social welfare regret and instability for a problem with $T$ rounds, $N$ users (at most $n$ of which are simultaneously active), $M$ items and length-$R$ feature vectors. Additionally, if the users are not allowed to reject the items, we show that the algorithm can achieve $\widetilde{\mathcal{O}} ( \sqrt{NMnRT} )$ regret and instability.
    \item For the more challenging setting where the provider has no contextual information about the items, we obtain \smash{$\widetilde{\mathcal{O}} ( ((N+M)R)^{1/4} (NMT)^{3/4})$}  regret and instability.
\end{itemize}

\textbf{Experiments:} We run experiments both on synthetic and real-world datasets to show the efficacy of the proposed algorithms. We demonstrate the significance of various aspects of our framework by contrasting our strategy with algorithms that fall short of capturing all of the components. Results show that the proposed algorithm can obtain significant improvements over these naive approaches.

\subsection{Related Work}

\textbf{Combinatorial Multi-Armed Bandits (CMAB):} The semi-bandits framework of \cite{audibert_2011} and the CMAB frameworks of \cite{chen_2013} and \cite{kveton_2015} model problems where a player chooses a subset of arms in each round and observes random rewards from the played arms. Since the allocations in our scenario correspond to selecting a subset of user-item pairs, it shows parallelism with these frameworks. However, the actions of the provider in our problem are not limited to the allocation, but also include the choice of accompanying pricing. Since the selection of prices affects social welfare and instability, it adds an additional degree of complexity to the problem of the provider, as we address in this work.

\textbf{Structured Linear Bandits:} The structured linear bandit frameworks of \cite{johnson_2016} and \cite{combes_2017} as well as low-rank linear bandit framework of \cite{lu_2021} offer methods to deal with the intrinsic structures between different actions in bandit problems. However, these works only concentrate on linear reward mechanisms (which would correspond to only observing the total reward of all allocations) while our approach has the additional ability to leverage the combinatorial nature of the reward feedback in our problem setting.

\textbf{Bandits in Economics:} Approaches based on bandits have also been applied to other economic contexts. One recent line of literature studies algorithms for learning socially-optimal matching in two-sided markets without transfers or incentives \citep{liu_2020, johari_2021}. \cite{jagadeesan_2021} considers learning stable one-to-one matching with transfers, but their model does not capture the effects of transfers on whether the users accept or reject the offered match. In the context of economics, bandit methods have also been applied to dynamic pricing \citep{kleinberg_2003}, incentivizing exploration \citep{frazier_2014} learning under competition \citep{aridor_2019}. 

\textbf{Pricing Mechanisms for Stable Outcomes:} 
In the literature of microeconomics, one of the standard approaches to discovering optimal allocations with stable prices is through a process called \emph{tâtonnement} (groping, trial-and-error) \citep{walras_2014}. During this process, each user repeatedly submits its demand at prices chosen by the provider while the provider successively adjusts the prices in response to the user's demand, so that capacity constraints are satisfied in equilibrium \citep{arrow_1951}. Since this equilibrium ensures that the items are allocated to users that will obtain the largest reward, it has found applications in various contexts such as auction design \citep{smith_1991} or bandwidth allocation over networks \citep{kelly_97}. However, these mechanisms assume that the users know their preferences for all possible items and require the users to repeatedly report their potential demands at many different prices. Since it might take many iterations until convergence, asking the users to repeatedly respond would create an unacceptable level of burden for them. Our approach circumvents both of these obstacles by learning the user preferences and making relevant recommendations.



\textbf{Recommendation with Capacity Constraints:} \cite{christakopoulou_2017} and \cite{makhijani_2019} use the notion of constrained resources to model and solve the problem of recommendation with capacity constraints. However, these works only consider optimizing the recommendation accuracy subject to capacity constraints without any consideration of the interactive learning mechanisms that discover user preferences through recommendations. The rotting bandit framework of \cite{levine_2017} can also model congestion in recommendation settings by allowing each arm’s reward to depend on the number of times it was played in the past. However, their model makes recommendations to sequentially arriving users instead of considering a market that contains multiple users that simultaneously interact with the provider.

\section{PROBLEM SETTING}
\label{sec:setting}

We use bold font for vectors $\mathbf{x}$ and matrices $\mathbf{X}$, and calligraphic font $\mathcal{X}$ for sets. For a vector $\mathbf{x}$, we denote its $i$-th entry by $x_i$. For a matrix $\mathbf{X}$, we denote its $(i,j)$-th entry by $X_{ij}$, $i$-th row by $[\vect{X}]_{i,:}$ and $j$-th column by $[\vect{X}]_{:,j}$. We denote the inner product of two matrices by $\langle \vect{A}, \vect{B} \rangle = \tr (\vect{A}^\mathrm{T} \vect{B})$ and the Hadamard (element-wise) product of two matrices by $\vect{A} \circ \vect{B}$.  We denote the $L_{2,2}$ (Frobenius) norm of a matrix by $\|\vect{A}\|_\text{F} = \sqrt{\langle \vect{A}, \vect{A} \rangle}$, and the $L_{2,\infty}$ norm of a matrix by $\|\vect{A}\|_{2, \infty} = \max_{i} \|[\vect{A}]_{i, :}\|_2$.

Suppose the market consists of a set of users $\mathcal{N}$ of size $N$ and a set of items $\mathcal{I}$ of size $M$. The items are allocated to the users in multiple rounds denoted by $t \in \mathbb{N}$. Allocation of an item $i \in \mathcal{I}$ to a user $u \in \mathcal{N}$ results in a random reward that has a distribution unknown to the system provider. The expected reward obtained from allocating item $i$ to user $u$ is denoted by $\Theta^*_{ui}$, and these values are collected into the reward matrix denoted by $\mathbf{\Theta}^* \in \mathbb{R}^{N \times M}$. 

We assume that each item has (time-varying) capacity that corresponds to the maximum number of users it can be allocated to. We denote the capacity of item $i \in \mathcal{I}$ at round $t$ by $c^{t}_{i}$, and collect them into vectors $\vect{c}_t \in \mathbb{R}^{M}$. Similarly, each user has a (time-varying) demand that corresponds to the maximum number of different items it can get allocated. We denote the demand of user $u \in \mathcal{N}$ at round $t$ by $d^{t}_{u}$, and collect them into vectors $\vect{d}_t \in \mathbb{R}^{N}$. Therefore, each item can only be allocated to at most $c^{t}_{i}$ different users, while each user can only get allocated at most $d^{t}_{u}$ different types of items in the round $t$. We shall call these the \emph{allocation constraints}. We also denote the set of active users at round $t$ by $\mathcal{N}_t := \{u \in \mathcal{N} : d^{t}_{u} > 0\}$ and denote the maximum number of simultaneously active users by $n := \max_{t} |\mathcal{N}_t|$.

Let $\mathbf{X}_{t}$ denote the \emph{allocation matrix} for round $t$ where the $(u, i)$-th entry is one if user $u$ is allocated item $i$ at round $t$, and zero otherwise. Due to the allocation constraints, all $\mathbf{X}_{t}$ must belong to the set of valid allocations defined as
\begin{equation*}
    \mathcal{X}_t := \{ \mathbf{X} \in \{0, 1\}^{N \times M} : \mathbf{X} \mathds{1} \leq \mathbf{d}_t \text{ and } \mathbf{X}^\textrm{T} \mathds{1} \leq \mathbf{c}_t\},
\end{equation*}
where the inequalities are entry-wise and $\mathds{1}$ denotes the all-ones vector of appropriate size.

\subsection{Optimal and Stable Allocations}
\label{sect_opt_allocations}

An allocation $\vect{X}$ is optimal if it maximizes the social welfare $\mathcal{W}(\vect{X}) := \langle \vect{X}, \mathbf{\Theta}^* \rangle$. 
Hence, if the provider had 
known the mean reward matrix $\vect{\Theta}^*$, the optimal allocation $\vect{X}^*_t$ at time $t$ would be obtained by solving the integer program
\begin{equation}
    \vect{X}^*_t \in \argmax_{\vect{X} \in \mathcal{X}_t} \; \langle \vect{X}, \vect{\Theta}^* \rangle .
    \label{integer_num}
\end{equation}
This integer program can be relaxed to a linear program by dropping the integral constraints (setting $X_{ui} \in [0, 1]$). Since the integrality gap of this relaxation is one \footnote{The integrality gap is the maximum ratio between the optimal values of the relaxation and integer program. (see Appendix \ref{appendix_num})}, the relaxed problem can be used to obtain a solution for \eqref{integer_num}. 

On the other hand, an allocation $\vect{X}$ is stable with prices $\vect{p}$ if for each individual user, the allocated set of items offers the maximum gains among all possible allocations.
Here, gains are measured in terms of the surplus they receive from consuming the allocated items at the given prices.
Thus, the notion of stability corresponds to the alignment between the market outcome and each user's preferences. Formally,
\begin{definition}
Let $\vect{x} \in \{0,1\}^M$ denote the allocation made to a user and let $\vect{\theta}_u^* = [\vect{\Theta}^*]_{u, :}$ denote the $u$-th row of the mean reward matrix. Then, the surplus of user $u$ from allocation $\vect{x}$ under prices $\vect{p}$ is given by
\begin{equation*}
    \psi_{u}(\vect{x}, \vect{p}) := \langle \vect{x}, \vect{\theta}_u^* - \vect{p} \rangle.
\end{equation*}
An allocation $\vect{X} \in \mathcal{X}_t$ with prices $\vect{p}$ is stable if every user's allocation maximizes their own surplus, i.e, 
\begin{equation}
\vect{x}_u \in \argmax_{\vect{x} \in \mathcal{X}^t_u} \psi_{u}(\vect{x}, \vect{p}),
\label{max_surplus}
\end{equation}
for all $u \in \mathcal{N}$, where $\vect{x}_u = [\vect{X}]_{u,:}$ is the $u$-th row of $\vect{X}$ and $\mathcal{X}_u^t := \{\vect{x} \in \{0,1\}^{M} : \mathds{1}^\mathrm{T} \vect{x} \leq d^t_u\}$ is the set of allocation vectors that satisfy the constraints of user $u$ at time $t$. 
\label{def:stability}
\end{definition}


A fundamental property of this market model shows that if the allocation and prices are set as the primal and dual optimal variables of the relaxed allocation problem, the market outcome becomes stable \citep{walras_2014}. To see this, we can consider the (partial) Lagrangian 
\begin{equation*}
\mathcal{L}(\vect{X}, \vect{p}; \vect{\Theta}) :=  \langle \vect{X}, \vect{\Theta} \rangle + \vect{p}^\textrm{T} ( \vect{c}_t - \vect{X}^\textrm{T} \mathds{1}),
\end{equation*}
where $\vect{p} \geq 0$ are the Lagrange multipliers associated with the capacity constraints. Then, the corresponding dual optimization problem is given by $\min_{\vect{p}  \geq 0} g(\vect{p}; \vect{\Theta}^*)$ where 
\begin{equation}
    g(\vect{p}; \vect{\Theta}) := \max_{\mathbf{X} \in [0, 1]^{N \times M}} \left \{ \mathcal{L}(\vect{X}, \vect{p}; \vect{\Theta}) \middle| \mathbf{X} \mathds{1} \leq \mathbf{d}_t \right\}
\label{eq:dual_function}
\end{equation}
is the dual function. Then, due to linear programming duality, the primal-dual optimal variables $(\vect{X}^*_t, \vect{p}^*_t)$ satisfy the stability condition given in \eqref{max_surplus}. Therefore, if the provider chooses the allocation $\vect{X}^*_t$ with prices $\vect{p}^*_t$, the market outcome is both optimum and stable. However, since the mean reward matrix $\vect{\Theta}^*$ is unknown to the provider, it cannot directly solve the allocation and pricing problem as given.

Thus, we consider an online framework where the provider makes sequential recommendations to users and the users provide feedback about their recommendations so that the provider can learn user preferences. For clarity of exposition, we first present the problem without allowing the users to reject their offers in Section \ref{section_formulate} and then extend our formulation to capture the user choices in Section \ref{section_ar}. 


\subsection{Problem Formulation}
\label{section_formulate}

In this section, we formulate the provider's problem and its objective. At each round $t$, the provider determines an allocation $\vect{X}_t \in \mathcal{X}_t$ along with a price vector $\vect{p}_t \in \mathbb{R}^{M}_{+}$. For convenience, we also define for each allocation $\vect{X}_t \in \mathcal{X}_t$ an equivalent set representation $\mathcal{A}_t \subseteq \mathcal{N} \times \mathcal{I}$ such that $(u,i) \in \mathcal{A}_t$ if and only if user $u$ is allocated item $i$. 

Since exact stability is unattainable when preferences are unknown, we need to formalize a measure of instability that captures how far we are from exact stability. Intuitively, the users want to obtain the largest possible surplus under the current condition of the market described by the prices $\vect{p}_t$. Therefore, the level of instability under prices $\vect{p}_t$ is related to the gap between the surplus the user obtained from allocation $\vect{X}_t$ and the maximum surplus the user could have obtained from any allocation. For this reason, we define the instability of a user as follows.
\begin{definition}
Given an allocation $\vect{X} \in \mathcal{X}_t$ with prices $\vect{p}$, let $\vect{x}_u = [\vect{X}]_{u,:}$ denote the allocation made to user $u$. Then, the maximum surplus of user $u$ at time $t$ is
\begin{equation*}
    \psi^{*}_{t, u}(\vect{p}) := \max_{\vect{x} \in \mathcal{X}_u^t} \psi_{u}(\vect{x}, \vect{p}),
\end{equation*}
and the instability of user $u$ at time $t$ is the difference
\begin{equation*}
    \mathcal{S}^{t}_{u}(\vect{x}_u, \vect{p}) := \psi^*_{t, u}(\vect{p}) - \psi_{u}(\vect{x}_u, \vect{p}).
\end{equation*}
\label{def:user_instability}
\end{definition}

\begin{remark}
Since $\psi^*_{t, u}(\vect{p}) \geq \psi_{u}(\vect{x}_u, \vect{p})$ for any $\vect{x}_u \in \mathcal{X}_u^t$, the instability $\mathcal{S}^{t}_{u}(\vect{x}_u, \vect{p})$ is always non-negative. Furthermore, we see that instability $\mathcal{S}^{t}_{u}(\vect{x}_u, \vect{p})$ is zero for all users $u \in \mathcal{N}$ if and only if the market outcome is (exactly) stable according to Definition \ref{def:stability}.
\end{remark}

After the provider decides on an allocation $\vect{X}_t$ with prices $\vect{p}_t$, it observes a random feedback $R^{t}_{ui}$ for each user-item pair $(u,i) \in \mathcal{A}_t$ that is recommended. We denote by $H_t$ the history $\{\mathbf{X}_{\tau}, \vect{p}_{\tau}, (R^{\tau}_{ui})_{(u, i) \in \mathcal{A}_{\tau}}\}_{\tau = 1}^{t-1}$ of observations available to the provider when choosing the next allocation $\mathbf{X}_{t}$ and prices $\vect{p}_t$. The provider employs a policy $\pi = \{ \pi_t | t \in \mathbb{N}\}$, which is a sequence of functions, each mapping the history $H_t$ to an action $(\mathbf{X}_{t}, \vect{p}_t)$. 

The task of the provider is to repeatedly allocate the items to the users and choose the prices such that it can achieve two goals simultaneously: maximum \emph{social welfare} and minimum \emph{instability}. To measure the performance of policies in achieving the objectives of maximum social welfare and minimum instability, we define the following metrics: 

\begin{definition}
For a policy $\pi$, its social-welfare regret in $T$ rounds is defined as
\begin{equation*} \label{cumulative_regret}
    \mathcal{R}^{SW}(T, \pi) =  \sum_{t = 1}^{T} \left[ \langle \mathbf{X}^*_t, \mathbf{\Theta}^*\rangle - \mathcal{W}_t(\vect{X}_t) \right],
\end{equation*}
and its instability in $T$ rounds is defined as
\begin{equation*} \label{cumulative_instability}
    \mathcal{R}^{I}(T, \pi) =  \sum_{t = 1}^{T} \sum_{u \in \mathcal{N}} \mathcal{S}^{t}_{u}(\vect{x}^t_u, \vect{p}_t),
\end{equation*}
where $\vect{X}^*_t \in \argmax_{\vect{X} \in \mathcal{X}_t} \; \langle \vect{X}, \vect{\Theta}^* \rangle$ denotes optimum allocation at time $t$ and $\vect{x}^t_u = [\vect{X}_t]_{u,:}$ is the $u$-th row of $\vect{X}_t$.
\label{def:regret}
\end{definition}

\subsection{Allowing Users to Accept or Reject the Offers}
\label{section_ar}

In this section, we model the users' ability to accept or reject the offers based on their valuations and the prices of the offered items. When an item $i$ is offered to user
$u$ at price $p_i$, we assume that the user accepts the item and the system obtains a positive reward if and only if the user expects to get a positive surplus from this consumption, i.e. $\Theta^*_{ui} \geq p_i$. Consequently, we modify the definition of the social welfare of an allocation $\vect{X}$ with prices $\vect{p}$ as
\begin{equation*}
     \mathcal{W}_t(\vect{X}_t) := \langle \mathbf{X}, \mathbf{\Theta}^* \circ \mathds{1}\{ \vect{\Theta}^* \geq \mathds{1} \vect{p}^\mathrm{T} \}\rangle
\end{equation*}
and modify the definition of the surplus of user $u$ as
\begin{equation*}
    \psi_{u}(\vect{x}, \vect{p}) := \langle \vect{x}, (\vect{\theta}_u^* - \vect{p}) \circ \mathds{1}\{ \vect{\theta}_u^* \geq \vect{p} \} \rangle.
\end{equation*}
The definitions of maximum surplus and instability of each user in Definition \ref{def:user_instability} as well as the definitions of social welfare regret and instability in Definition \ref{def:regret} are also modified accordingly. Note that these modifications do not affect the feedback model and the provider is assumed to observe the random feedback $R^t_{ui}$ for each recommendation pair $(u,i) \in \mathcal{A}_t$ regardless of their acceptance.




\section{METHODOLOGY}
\label{sect_methodology}

In order to facilitate our analysis, we start with assumptions that are standard in the online learning literature.
\begin{assumption} 
For all $(u, i) \in \mathcal{N} \times \mathcal{I}$ and $t \in \mathbb{N}$, the rewards $R^{t}_{ui}$ are independent and $\eta$-sub-Gaussian with mean $\Theta^*_{ui} \in [0, 1]$, i.e., $\mathds{E}[\exp(\lambda (R^{t}_{ui} - \Theta^*_{ui}))] \leq \exp(\lambda \eta^2/2)$ almost surely for any $\lambda$.
\label{rew_assumptio}
\end{assumption}

As is common in recommendation settings, we model the correlation between users and items through a linear latent factor model. 
Each user $u$ (item $i$) is associated with a feature vector $\vect{f}_u \in \mathbb{R}^{R}$ ($\vect{\phi}_i \in \mathbb{R}^{R}$) in a shared $R$-dimensional space (typically $R \ll  N, M$), and the mean reward of each user-item allocation pair is given by a linear form.  

\begin{assumption}
\label{low_assum}
For all $(u, i) \in \mathcal{N} \times \mathcal{I}$, and $t \in \mathbb{N}$, the expected reward is given by
$\Theta^*_{ui} = \langle \vect{f}_u, \vect{\phi}_i \rangle$ with feature vectors satisfying $\|\vect{f}_u\|_2 \leq 1$ and $\|\vect{\phi}_i\|_2 \leq 1$.
\end{assumption}

In order to make use of initial historical data, we assume that the algorithm has access to an initial estimate $\vect{\Theta}_\circ$ that satisfies $\|\vect{\Theta}_\circ - \vect{\Theta}^* \|_\text{F} \leq G$. Note that one can also set $\vect{\Theta}_\circ = \vect{0}$ and let $G$ be a number satisfying $\|\vect{\Theta}^* \|_\text{F} \leq G$.

\subsection{Contextual Interactive Learning for Allocation and Pricing (CX-ILAP)}

We start our analysis by considering the case where the feature vectors $\vect{\phi}_i$ for the items $i \in \mathcal{I}$ are given as contextual information, but the feature vectors $\vect{f}_i$ for the users $u \in \mathcal{N}$ are unknown. Letting $\vect{\Phi} \in \mathbb{R}^{M \times R}$ be the matrix with rows $\vect{\phi}_i$, this assumption corresponds to asserting that the mean reward matrix $\vect{\Theta}^*$ belongs to the set
\begin{equation*}
\mathcal{F} := \{ \vect{\Theta} : \exists \vect{F} \in \mathbb{R}^{N \times R}, \vect{\Theta} = \vect{F} \vect{\Phi}^\mathrm{T}, \|\vect{F}\|_{2, \infty} \leq 1\}.
\end{equation*}

\begin{algorithm}[t]
\caption{\small{Interactive Learning for Allocation and Pricing}}
\begin{algorithmic}[1]
\Require parameters $T, \delta, \alpha, \gamma, \nu$, and initial estimate $\vect{\Theta}_\circ$
\For{$t = 1, 2, \dots, T$}
\IndentLineComment Find the regularized least squares estimate:
\State $ \widehat{\vect{\Theta}}_t = \argmin_{\vect{\Theta} \in \mathcal{F}} \left \{ L_{2,t}(\vect{\Theta}) + \gamma \|\vect{\Theta} - \vect{\Theta}_\circ \|_{\mathrm{F}}^2 \right \}$ \label{alg:line:least_squares}
\LineComment Construct the confidence set:
\State $\mathcal{C}_t := \{ \vect{\Theta} \in \mathcal{F} : \|\vect{\Theta} - \widehat{\vect{\Theta}}_t \|_{E^t_{2,\infty}} \leq \sqrt{\rho_t(\delta, \alpha, \gamma)} \}$
\LineComment Compute the OFU allocation:
\State $(\vect{X}_t, \vect{\Theta}_t) \in \argmax_{(\vect{X}, \vect{\Theta}) \in \mathcal{X}_t \times \mathcal{C}_t} \; \langle \vect{X}, \vect{\Theta} \rangle$
\LineComment Compute the prices:
\State $\bar{\vect{p}}_t \in \argmin_{\vect{p} \geq 0} g(\vect{p}; \vect{\Theta}_t)$
\State $\vect{p}_t = \bar{\vect{p}}_t - \nu \sqrt{w_t}$ for $w_t$ defined in \eqref{w_t_def}
\State Offer the allocation $\vect{X}_t$ together with prices $\vect{p}_t$
\State Observe $R^{t}_{ui}$ for all $(u, i) \in \mathcal{A}_{t}$
\EndFor
\end{algorithmic}
\label{alg_low}
\end{algorithm}

Our method, summarized in Algorithm \ref{alg_low}, follows the standard OFU (Optimism in the Face of Uncertainty) principle of  \cite{abbasi_2011}. It maintains a confidence set $\mathcal{C}_t$ which contains the true parameter $\vect{\Theta}^*$ with high probability and chooses the allocation $\vect{X}_t$ according to
\begin{equation}
(\vect{X}_t, \vect{\Theta}_t) \in \argmax_{(\vect{X}, \vect{\Theta}) \in \mathcal{X}_t \times \mathcal{C}_t} \; \langle \vect{X}, \vect{\Theta} \rangle.
\label{low_oful}
\end{equation}
Then, in order to construct its pricing strategy, it computes the optimal dual variables $\bar{\vect{p}}_t$ for the allocation problem described by $\vect{\Theta}_t$, i.e., it solves the dual problem
\begin{equation}
    \bar{\vect{p}}_t \in \argmin_{\vect{p} \geq 0} g(\vect{p}; \vect{\Theta}_t),
\label{price_selection}
\end{equation} 
where the dual function is as defined in \eqref{eq:dual_function}. Typically, the faster the confidence set $\mathcal{C}_t$ shrinks, the lower regret and instability we have. 
However, the main difficulty is to construct a series of $\mathcal{C}_t$ that leverage the combinatorial observation model 
as well as the structure of the mean reward matrix 
so that we have low regret bounds.
In this work, we consider constructing confidence sets that are centered around the regularized least square estimates. We let the cumulative squared prediction error of $\vect{\Theta}$ at time $t$ be
\begin{equation*}
    L_{2,t}(\vect{\Theta}) = \sum_{\tau=1}^{t-1} \sum_{(u, i) \in \mathcal{A}_\tau} (\Theta_{ui}  - R^{\tau}_{ui})^2,
\end{equation*}
and calculate the regularized least squares estimate as given in line~\ref{alg:line:least_squares} of Algorithm \ref{alg_low}.
Then, the confidence sets take the form $\mathcal{C}_t := \{ \vect{\Theta} \in \mathcal{F} : \|\vect{\Theta} - \widehat{\vect{\Theta}}_t \|_{E^t_{2,\infty}} \leq \sqrt{\rho_t} \}$, where $\rho_t$ is an appropriately chosen confidence parameter, and $\| \cdot \|_{E^t_{2,\infty}}$ is the regularized empirical $L_{2, \infty}$ norm, namely,
\begin{equation*}
    \| \vect{\Delta} \|_{E^t_{2,\infty}}^2 := \max_{u \in \mathcal{N}} \sum_{i\in \mathcal{I}} (n^{t}_{ui} + \gamma) \Delta_{ui}^2,
\end{equation*}
where $n^{t}_{ui} := \sum_{\tau=1}^{t-1} \mathds{1} \{(u,i) \in \mathcal{A}_\tau\}$ is the number of times item $i$ has been allocated to user $u$ before time $t$ (excluding time $t$).
Hence, the empirical norm is a measure of discrepancy that weighs the entries depending on how much they have been explored. 
Roughly speaking, since the confidence region is wider in directions that are not yet well-explored, the OFU step described in \eqref{low_oful} is more inclined to make allocations that include the corresponding user-item pairs. 
In order to obtain low-regret guarantees for the allocations, the first step is to choose the correct $\rho_t$ parameter such that $\mathcal{C}_t$ will contain the true parameter $\vect{\Theta}^*$ for all $t$ with high probability. 
The following Lemma establishes that, if we set $\rho_t = \rho_t^*(\delta, \alpha, \gamma)$, the resulting confidence sets have the desired properties.
\begin{lemma} For any $\delta, \alpha, \gamma > 0$, if the confidence sets are
\begin{equation}
    \mathcal{C}_t := \left\{ \vect{\Theta} \in \mathcal{F} : \|\vect{\Theta} - \widehat{\vect{\Theta}}_t \|_{E^t_{2,\infty}} \leq \sqrt{\rho_t^*(\delta, \alpha, \gamma)} \right\}
    \label{conf_sets_context}
\end{equation}
with the confidence parameter
\begin{align*}
    \rho_t^*(\delta, \alpha, \gamma) &:= 8 \eta^2 R \log \left(3 N / (\alpha \delta) \right) + 4 \gamma G^2 \\
    &+ 2 \alpha t \sqrt{M} \left [ 8 + \sqrt{8 \eta^2 \log(4MNt^2/\delta)} \right],  
\end{align*}
then, with probability at least $1 - 2 \delta$, $\mathcal{C}_t \ni \vect{\Theta}^*$, for all $t$ .
\label{conf_sets_lemma_context}
\end{lemma}

If the users do not have the option to reject the recommendations, then the algorithm can offer the items at the estimated equilibrium price by setting $\vect{p}_t = \bar{\vect{p}}_t$, and it can provably achieve low regret and instability. However, when the users are allowed to reject the offers based on their expected surplus, such a pricing strategy cannot achieve provable optimality guarantees for the social welfare since it does not alleviate the possibility of rejections. 
Instead, our approach is to reduce the likelihood of rejections by choosing the prices as $\vect{p}_t = \bar{\vect{p}}_t - \nu \sqrt{w_t}$, 
where $w_t$ is the \emph{confidence width} of the allocation $\vect{X}_t$, defined as
\begin{equation}
    w_t := \sum_{(u,i) \in \mathcal{A}_t} (n^{t}_{ui} + \gamma)^{-1},
    \label{w_t_def}
\end{equation}
and $\nu > 0$ is a parameter that depends on the problem size. As we show in the appendix, this pricing strategy can guarantee the number of rejections to grow only sublinearly in $T$, allowing us to achieve sub-linear regret and instability even when the users are allowed to reject.

Finally, we can show that if the algorithm follows the OFU allocations in \eqref{low_oful} and prices in \eqref{price_selection}, while constructing the confidence sets according to \eqref{conf_sets_context}, it obtains the following overall regret guarantee.
\begin{theorem}\label{thm_alloc_regret}
Under Assumptions \ref{rew_assumptio} and \ref{low_assum} with known item feature matrix $\vect{\Phi}$, set $\nu = \left( 4 \rho_T^*(\delta, \alpha, \gamma)/(n M^2) \right)^{1/4}$, $\delta > 0$, $\alpha > 0$ and $\gamma \geq 1$. Then, with probability $1 - 2\delta$, the social welfare regret and instability of Algorithm \ref{alg_low} satisfy
\begin{align*}
    \mathcal{R}^{SW}(T, \pi) &\leq \sqrt{\kappa_T nM T} + (\kappa_T)^{\frac{1}{4}} (nMT)^{\frac{3}{4}}, \\
    \mathcal{R}^{I}(T, \pi) &\leq \sqrt{\kappa_T nM T} + (\kappa_T)^{\frac{1}{4}} (nMT)^{\frac{3}{4}},
\end{align*}
where $\kappa_T = 8 N \rho_T^*(\delta, \alpha, \gamma) \log (1 + T/\gamma)$. Instead, if we assume that the users are not allowed to reject the allocations and set $\nu = 0$, we have $\mathcal{R}^{SW}(T, \pi) \leq \sqrt{\kappa_T n M T}$ and $\mathcal{R}^{I}(T, \pi) \leq \sqrt{\kappa_T n M T}$, with probability $1 - 2\delta$.
\end{theorem}

Thus, by correctly selecting the parameters of the algorithm, we can simultaneously achieve social welfare regret and instability of order \smash{$\widetilde{\mathcal{O}} ( (NR)^{1/4} (nMT)^{3/4})$}, where \smash{$\widetilde{\mathcal{O}}$} is the big-O notation ignoring the poly-logarithmic factors of $N, M, T, R$. On the other hand, if the users are not given the option to reject the offers, then we can obtain social welfare regret and instability of order $\widetilde{\mathcal{O}} ( \sqrt{NMnRT} )$. 

If we were to ignore the structure between the mean rewards of user-item pairs and apply the standard combinatorial bandit algorithms to decide on the allocations, the social welfare regret (in the case where users cannot reject the offers) would be of order \smash{$\widetilde{\mathcal{O}} ( M \sqrt{N n T} )$} \citep{kveton_2015}. Since $n, R \ll M, N$ in many applications of our consideration, our algorithm significantly outperforms this naive approach by leveraging the structure of the rewards.



\subsection{Low-Rank Interactive Learning for  Allocation and Pricing (LR-ILAP)}

\label{sect_lrcb}

In this section, we consider removing the assumption that the provider has access to the feature vectors $\vect{\phi}_i$ of the items. In order to use the structure imposed by Assumption \ref{low_assum} without having the knowledge of any of the factors $\vect{\phi}_i$ or $\vect{f}_i$, we can write that $\vect{\Theta}^*$ belongs to the set
\begin{equation}
    \mathcal{L} = \{ \vect{\Theta} \in \mathbb{R}^{N \times M} : \|\vect{\Theta}\|_{\infty} \leq 1, \text{rank}(\vect{\Theta}) \leq R\}.
    \label{low_l}
\end{equation}

To adapt Algorithm \ref{alg_low} for this setting, we modify the construction of our confidence regions. In order to take advantage of the structure of the arms, we let $\mathcal{N}(\mathcal{F}, \alpha, \| \cdot \|_{\text{F}})$ denote the $\alpha$-covering number of $\mathcal{L}$ in the Frobenious-norm $\| \cdot \|_{\text{F}}$, and define empirical $L_2$ norm at time $t$ as
\begin{equation*}
    \| \vect{\Delta} \|_{E_{2}^t}^2 := \sum_{u \in \mathcal{N}} \sum_{i\in \mathcal{I}} (n^{t}_{ui} + \gamma) \Delta_{ui}^2.
\end{equation*}
Then, similar to Lemma \ref{conf_sets_lemma_context}, we can show that, with probability $1-2\delta$, $\vect{\Theta}^*$ belongs to the confidence region 
\begin{equation}
    \mathcal{Q}_t := \left\{ \vect{\Theta} \in \mathcal{L} : \|\vect{\Theta} - \widehat{\vect{\Theta}}_t \|_{E_2^t} \leq \sqrt{\beta_t^*  (\delta, \alpha, \gamma)} \right\}
    \label{conf_sets_lr},
\end{equation}  
where $ \widehat{\vect{\Theta}}_t = \argmin_{\vect{\Theta} \in \mathcal{L}} \left \{ L_{2,t}(\vect{\Theta}) + \gamma \|\vect{\Theta} - \vect{\Theta}_\circ \|_{\mathrm{F}}^2 \right \}$ and the confidence parameter is 
\begin{align*}
    \beta_t^*(\delta, \alpha, \gamma) & := 8 \eta^2 \log \left(\mathcal{N}(\mathcal{L}, \alpha, \| \cdot \|_{\text{F}}) / \delta \right) + 4 \gamma G^2 \\
    &+ 2 \alpha t \sqrt{NM} \left [ 8 + \sqrt{8 \eta^2 \log(4NM t^2/\delta)} \right].
\end{align*}

Then, we show that the covering number for $\mathcal{L}$ given in equation \eqref{low_l} is upper bounded: $\log \mathcal{N}(\mathcal{L}, \alpha, \| \cdot \|_{\text{F}}) \leq (N + M + 1) R \log ( 9 \sqrt{NM} / \alpha )$. 
Consequently, the following theorem shows that if the algorithm constructs the confidence sets according to \eqref{conf_sets_lr}, it can achieve $\widetilde{\mathcal{O}}(((N+M)R)^{1/4} (NMT)^{3/4})$ social welfare regret and instability. Moreover, if the users are not given the option to reject our offers, then the order improves to $\widetilde{\mathcal{O}} ( \sqrt{NM(N+M)RT})$.

\begin{theorem}\label{thm_alloc_regret_lr}
Under Assumptions \ref{rew_assumptio} and \ref{low_assum} with unknown feature matrices, set $\nu = \left( 4 \beta_T^*(\delta, \alpha, \gamma)/(N^2 M^2) \right)^{1/4}$, $\delta>0$,  $\alpha>0$, $\gamma \geq 1$ and construct the confidence sets according to \eqref{conf_sets_lr}. Then, with probability $1 - 2\delta$, the social welfare regret and instability of Algorithm \ref{alg_low} satisfy
\begin{align*}
    \mathcal{R}^{SW}(T, \pi) &\leq \sqrt{\lambda_T NM T} + (\lambda_T)^{\frac{1}{4}} (NMT)^{\frac{3}{4}}, \\
    \mathcal{R}^{I}(T, \pi) &\leq \sqrt{\lambda_T NM T} + (\lambda_T)^{\frac{1}{4}} (NMT)^{\frac{3}{4}},
\end{align*}
where $\lambda_T = 8 \beta_T^*(\delta, \alpha, \gamma) \log (1 + T/\gamma)$. Instead, if we assume that the users are not allowed to reject our allocations and set $\nu = 0$, we have $\mathcal{R}^{SW}(T, \pi) \leq \sqrt{\lambda_T NM T}$ and $\mathcal{R}^{I}(T, \pi) \leq \sqrt{\lambda_T NM T}$, with probability $1 - 2\delta$.
\end{theorem}
\section{EXPERIMENTS}
\label{sect_exp}

\begin{figure*}[th]
\vspace{-10pt}
\center
\includegraphics{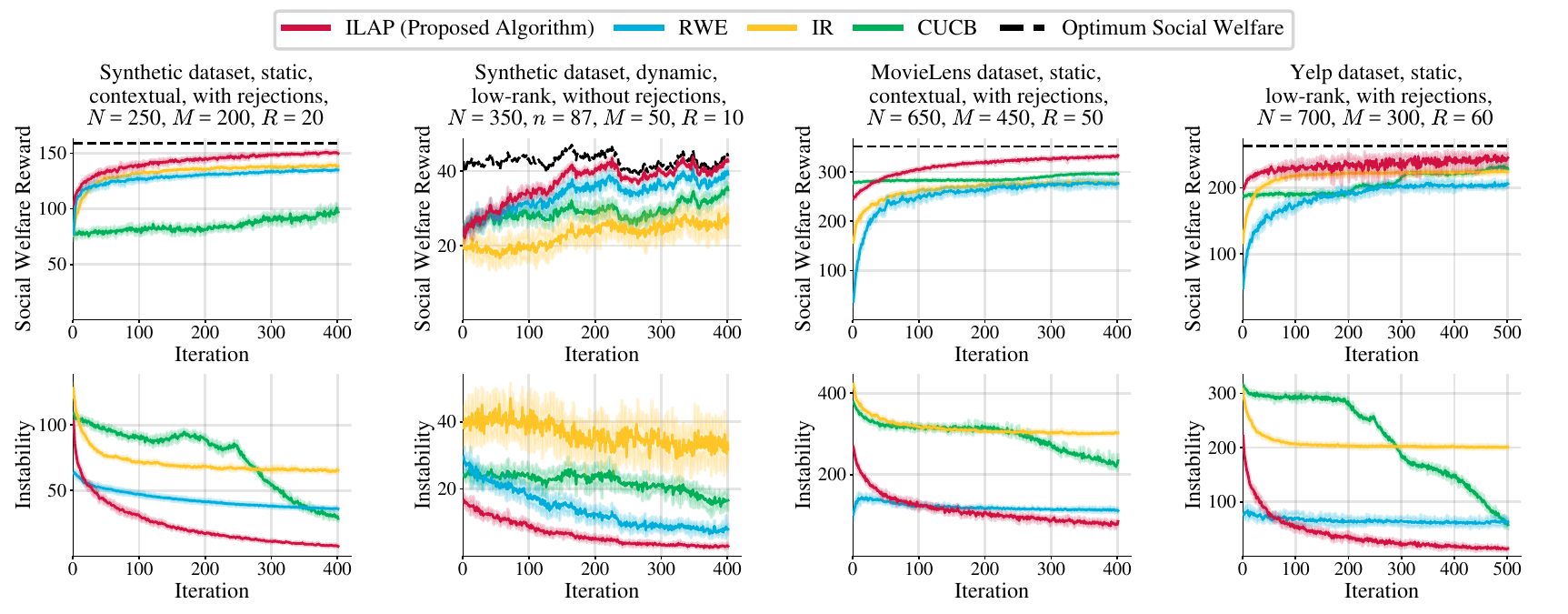}
\vspace{-20pt}
\caption{Instantaneous reward and instability at each round in different experimental settings. The experiments are run on $20$ problem instances and means are reported together with one standard deviation of uncertainty regions.}
\label{fig_regrets}
\vspace{-10pt}
\end{figure*}

In this section, we demonstrate the efficacy of our proposed algorithm through an experimental study over synthetic and real-world datasets. The goal of our experimental evaluation is to evaluate our algorithm for making online recommendations in various market settings.

\textbf{Baseline Algorithms:} To the best of our knowledge, there are no current approaches specifically designed to make online recommendations and allocations considering the constraints of the market. 
Therefore, we show the importance of different aspects in our framework by comparing it with algorithms that do not capture all these aspects: a non-bandit algorithm that only aims for momentary performance, a bandit-based algorithm that is unaware of the capacities, and another bandit-based algorithm that is unaware of the correlation between the arms.

\textbf{(1) RWE (Recommendations without Exploration):}  It solves for the least squares problem given in line~\ref{alg:line:least_squares} of Algorithm \ref{alg_low} to estimate the mean rewards of user-item allocation pairs. Then, without deliberate exploration, it decides on the allocation and prices as the optimum values for the allocation problem with the least square estimation.

\textbf{(2) IR (Interactive Recommendations):}  It runs the interactive recommendation algorithm of \cite{zhao_2013} that finds the best items for each user without considering the capacity constraints. Being oblivious of the capacities, it also does not set any price for the items. Since these recommendations do not necessarily satisfy the constraints, it needs to perform a post-processing step: If it decides to recommend an item to more users than its capacity, it only chooses a random subset of the assigned users. Heuristically, we let the algorithm observe zero rewards ($R^{t}_{ui} = 0$) for all the user-item allocations that are eliminated during this post-processing step because otherwise, the allocations would continue to violate the capacity constraints and the algorithm would not reduce its regret.

\textbf{(3) CUCB (Combinatorial-UCB):}  It runs the CUCB algorithm of \cite{chen_2013} to decide on allocations without assuming any structure between the users and items. It views the user-item allocation pairs as arms that have no correlation between them. At each round, it decides on the allocation and prices by solving the allocation problem using the upper confidence bounds of the reward parameters.

\subsection{Experimental Setup and Datasets} 

We use a synthetic dataset and two real-world datasets. For the synthetic data, we set the mean reward values $\Theta^*_{ui} = \langle \vect{f}_u, \vect{\phi}_i \rangle$ by randomly choosing length-R feature vectors that satisfy $\|\vect{f}_u\|_2 \leq 1$, for all $u \in \mathcal{N}$, and $\|\vect{\phi}_i\|_2 \leq 1$, for all $i \in \mathcal{I}$. For the real-world data, we consider the following publicly available data sets: a subset of Movielens 100k \citep{harper_2015} which includes ratings from 650 users on 450 movies and a subset of Yelp Dataset \citep{yelp_dataset} that corresponds to ratings from 700 users on 300 restaurants. The subsets of users and items are selected such that the datasets contain ratings for most of the user-item pairs. To set the problem up, the missing ratings are completed with an off-the-shelf matrix completion algorithm with matrix factorization \citep{hastie_2015}. We consider problem settings with known item features as well as with unknown features. In the setting with known features for the items (contextual), all algorithms have access to the item feature matrix $\vect{\Phi}$. In the other case (low-rank), none of the algorithms are provided with any contextual information. To the best of our knowledge, the capacity information is not available in any of the publicly available recommendation datasets. Therefore, we consider instantiating random capacities. We consider settings with static and dynamic (time-varying) capacities/demands. For the static case, we assume that all users request one item at all rounds, and the capacity of each item remains unchanged with time. In the dynamic setting, we allow both the demands $\vect{d}_t$ and capacities $\vect{c}_t$ vary with time $t$. We consider that each entry of $\vect{d}_t$ is independently sampled from a distribution over $\{0, 1\}$ so that each user $u$ is either active ($d^t_u = 1$) or inactive ($d^t_u = 0$). Similarly, each entry of $\vect{c}_t$ is independently sampled from a distribution over non-negative integers. Lastly, the rewards are assumed to have distribution $R^{t}_{ui} \sim \mathcal{N}(\Theta_{ui}^*, \eta^2)$. For additional details on the experimental setup, please refer to Appendix \ref{sect_additional_exp}. 

\subsection{Results} 

We provide a summary of our results in Figure \ref{fig_regrets}. The results for other experimental settings are left to Appendix \ref{sect_additional_exp}. The observations can be summed up into the following points: (1) ILAP algorithm (our proposed approach) is able to achieve lower regret and instability than all other baseline methods in all experimental settings. (2) Even though RWE can obtain slightly larger rewards than our method in the initial rounds, it often gets stuck at suboptimal allocations. So, it cannot achieve \emph{no-regret} because it tries to directly exploit the information it acquired so far without making any deliberate explorations. (3) Since IR does not consider the capacities while making the allocations, the recommendations exceed the capacities of the items and lead to very large regrets. Furthermore, since it does not set any prices for the items, it incurs high instability. (4) Since CUCB does not leverage the low-rank structure of the parameters, it needs to sample and learn about each user-item allocation pair separately. Hence, it takes much longer for it to learn the optimum allocations.
\section{CONCLUSION}


In this paper, we have studied the setting of interactive recommendations that achieve socially optimal and stable allocations under capacity constraints. 
We considered an online learning setting where the users are given sequential recommendations which they can accept or reject based on the prices determined by the provider. 
To leverage the correlation between different user-item pairs, we considered linear latent factor models with partially-known and unknown features. We have proposed a sequential allocation and pricing algorithm that enjoys sub-linear social welfare regret and sub-linear instability.

\bibliographystyle{apalike}
\bibliography{Bib_Database}

\clearpage
\onecolumn
\aistatstitle{Interactive Learning with Pricing for \\ Optimal and Stable Allocations in Markets:\\
Supplementary Materials}
\thispagestyle{empty}

\appendix

\section{RELAXATION OF INTEGER PROGRAM}
\label{appendix_num}

A traditional linear integer program (IP) in matrix form is formulated as
\begin{align}
\begin{split}
    \max_{\mathbf{x}} &\; \mathbf{t}^\mathrm{T} \mathbf{x}\\
    \text{s.t.} &\; \mathbf{A} \mathbf{x} \leq \mathbf{b}\\
    &\; \mathbf{x} \in \mathds{Z}_+^d\\
    \label{int_prog}
\end{split}
\end{align}

This problem can be relaxed to a linear program by dropping the integral constraints (setting $\mathbf{x} \in \mathds{R}_+^d$). The integrality gap of an integer maximization program is defined as the maximum ratio of the optimal values for the relaxed linear program and the integer program. When the vector $\mathbf{b}$ is integral and the matrix $\mathbf{A}$ is totally unimodular (all entries are 1, 0, or -1 and every square sub-minor has determinant of +1 or -1) then the integrality gap is one and the solution of the relaxed linear program is integer valued \citep{bertsekas_1991}. 

Hence, we can solve \eqref{int_prog} by instead solving the following relaxed linear program:
\begin{align}
\begin{split}
    \max_{\mathbf{x}} &\; \mathbf{t}^\mathrm{T} \mathbf{x}\\
    \text{s.t.} &\; \mathbf{A} \mathbf{x} \leq \mathbf{b}\\
    &\; \mathbf{x} \in \mathds{R}_+^d\\
\end{split}
\end{align}

For a matrix $\vect{A}$ whose rows can be partitioned into two disjoint sets $\mathcal{C}$  and $\mathcal{D}$, the following four conditions together are sufficient for $\vect{A}$ to be totally unimodular \citep{heller_1957}:
\begin{enumerate}[nosep, labelindent= 2pt, align= left, labelsep=0.4em, leftmargin=*]
    \item Every entry in $\vect{A}$ is $0$, $+1$, or $-1$.
    \item Every column of $\vect{A}$ contains at most two non-zero entries.
    \item If two non-zero entries in a column of $\vect{A}$ have the same sign, then the row of one is in $\mathcal{C}$  , and the other in $\mathcal{D}$ .
    \item If two non-zero entries in a column of $\vect{A}$ have opposite signs, then the rows of both are in $\mathcal{C}$  , or both in $\mathcal{D}$ .
\end{enumerate}

In the setting of resource allocation, we can write problem \eqref{integer_num} equivalently as problem \eqref{int_prog} where $\mathbf{x} = \text{vec}(\mathbf{X})$, $\mathbf{t} = \text{vec}(\mathbf{\Theta^*})$, $\mathbf{A}$ and $\mathbf{b}$ are given as
\begin{equation}
    \mathbf{A} = 
    \begin{bmatrix}
    \mathds{1}_N^\mathrm{T} \otimes \mathbf{I}_M\\
    \mathbf{I}_N \otimes \mathds{1}_M^\mathrm{T}\\
    \end{bmatrix}
    \qquad 
    \mathbf{b} = 
    \begin{bmatrix}
    \vect{c}_t\\
    \vect{d}_t\\
    \end{bmatrix}
\label{def_Ab}
\end{equation}

For matrix $\vect{A}$ given in \eqref{def_Ab}, we can set $\mathcal{C}$ to be the set of first $M$ rows corresponding to the capacity constraints, and $\mathcal{D}$ to be the set of remaining rows corresponding to the demand constraints. Since this $\vect{A}$ matrix satisfies the conditions of the proposition for sets $\mathcal{C}$  and $\mathcal{D}$, we obtain that $\mathbf{A}$ is totally unimodular. Finally, since the vector $\mathbf{b}$ is integral and the matrix $\mathbf{A}$ is totally unimodular, the integrality gap is one.

\section{IMPLEMENTATION VIA ALTERNATING OPTIMIZATION}

The following algorithm describes an efficient implementation of CX-ILAP algorithm.

\begin{algorithm}[H]
\caption{CX-ILAP with Alternating Optimization}
\begin{algorithmic}
\Require horizon $T$, initial estimate $\vect{\Theta}_{\circ}$ with $\|\vect{\Theta}_{\circ} - \vect{\Theta}^* \|_\mathrm{F} \leq G$, parameters $\delta, \alpha, \nu > 0$, $\gamma \geq 1$.
\For{$t = 1, 2, \dots, T$}
\State randomly initialize $\widehat{\vect{F}}$ and $\widehat{\vect{\Phi}}$
\State $\widehat{\vect{F}} \gets \argmin_{\vect{F} \in \mathbb{R}^{N \times R}} \left \{ \sum_{\tau=1}^{t-1} \sum_{(u, i) \in \mathcal{A}_\tau} (\vect{f}_u^\textrm{T} \vect{\phi}_i  - R^{\tau}_{ui})^2 + \gamma \|\vect{F} \vect{\Phi}^\textrm{T} - \vect{\Theta}_{\circ} \|_\mathrm{F}^2 \right \}$
\State $\widehat{\vect{\Theta}}_t \gets \widehat{\vect{F}} \vect{\Phi}^\textrm{T}$
\State $\vect{X} \gets \mathbb{1}_{N \times M}$, $\vect{F} \gets \widehat{\vect{F}}$, $\vect{\Phi} \gets \widehat{\vect{\Phi}}$
\While{convergence criterion not satisfied}
\State $\vect{F} \gets \argmax_{\vect{F} \in \mathbb{R}^{N \times R}} \langle \vect{X}, \vect{F} \vect{\Phi}^\textrm{T} \rangle$ s.t. $\|\vect{F} \vect{\Phi}^\textrm{T} - \widehat{\vect{\Theta}}_t \|_{E^t_{2, \infty}} \leq \sqrt{ \rho_t^*(\delta, \alpha, \gamma)}$

\State $\vect{\Theta} \gets \vect{F} \vect{\Phi}^\textrm{T}$
\State $(\vect{X}, \bar{\vect{p}}) \gets$ optimal primal-dual variables for the problem $\max_{\vect{X} \in \mathcal{X}_t} \; \langle \vect{X}, \vect{\Theta} \rangle$

\EndWhile
\State $(\vect{X}_t, \bar{\vect{p}}_t) \gets (\vect{X}, \bar{\vect{p}})$
\State $\vect{p}_t = \bar{\vect{p}}_t - \nu \sqrt{w_t}$ for $w_t$ defined in \eqref{w_t_def}
\State Offer the allocation $\vect{X}_t$ together with prices $\vect{p}_t$
\State Observe $R^{t}_{ui}$ for all $(u, i) \in \mathcal{A}_{t}$
\EndFor
\end{algorithmic}
\label{alg_mf_cx}
\end{algorithm}

The following algorithm describes an efficient implementation of LR-ILAP algorithm using matrix factorization.

\begin{algorithm}[H]
\caption{LR-ILAP with Alternating Optimization using Matrix Factorization}
\begin{algorithmic}
\Require horizon $T$, initial estimate $\vect{\Theta}_{\circ}$ with $\|\vect{\Theta}_{\circ} - \vect{\Theta}^* \|_\mathrm{F} \leq G$, parameters $\delta, \alpha, \nu > 0$, $\gamma \geq 1$.
\For{$t = 1, 2, \dots, T$}
\State randomly initialize $\widehat{\vect{F}}$ and $\widehat{\vect{\Phi}}$
\While{convergence criterion not satisfied}
    \State $\widehat{\vect{F}} \gets \argmin_{\vect{F} \in \mathbb{R}^{N \times R}} \left \{ \sum_{\tau=1}^{t-1} \sum_{(u, i) \in \mathcal{A}_\tau} (\vect{f}_u^\textrm{T} \vect{\phi}_i  - R^{\tau}_{ui})^2 + \gamma \|\vect{F} \widehat{\vect{\Phi}}^\textrm{T} - \vect{\Theta}_{\circ} \|_\mathrm{F}^2 \right \}$
    \State $\widehat{\vect{\Phi}} \gets \argmin_{\vect{\Phi} \in \mathbb{R}^{M \times R}} \left \{ \sum_{\tau=1}^{t-1} \sum_{(u, i) \in \mathcal{A}_\tau} (\vect{f}_u^\textrm{T} \vect{\phi}_i  - R^{\tau}_{ui})^2 + \gamma \|\widehat{\vect{F}} \vect{\Phi}^\textrm{T} - \vect{\Theta}_{\circ} \|_\mathrm{F}^2 \right \}$
\EndWhile
\State $\widehat{\vect{\Theta}}_t \gets \widehat{\vect{F}} \widehat{\vect{\Phi}}^\textrm{T}$
\State $\vect{X} \gets \mathbb{1}_{N \times M}$, $\vect{F} \gets \widehat{\vect{F}}$, $\vect{\Phi} \gets \widehat{\vect{\Phi}}$
\While{convergence criterion not satisfied}
\While{convergence criterion not satisfied}
\State $\vect{F} \gets \argmax_{\vect{F} \in \mathbb{R}^{N \times R}} \langle \vect{X}, \vect{F} \vect{\Phi}^\textrm{T} \rangle$ s.t. $\|\vect{F} \vect{\Phi}^\textrm{T} - \widehat{\vect{\Theta}}_t \|_{E^t_2} \leq \sqrt{ \beta_t^*(\delta, \alpha, \gamma)}$
\State $\vect{\Phi} \gets \argmax_{\vect{\Phi} \in \mathbb{R}^{M \times R}} \langle \vect{X}, \vect{F} \vect{\Phi}^\textrm{T} \rangle$ s.t. $\|\vect{F} \vect{\Phi}^\textrm{T} - \widehat{\vect{\Theta}}_t \|_{E^t_2} \leq \sqrt{ \beta_t^*(\delta, \alpha, \gamma)}$
\EndWhile

\State $\vect{\Theta} \gets \vect{F} \vect{\Phi}^\textrm{T}$
\State $(\vect{X}, \bar{\vect{p}}) \gets$ optimal primal-dual variables for the problem $\max_{\vect{X} \in \mathcal{X}_t} \; \langle \vect{X}, \vect{\Theta} \rangle$

\EndWhile
\State $(\vect{X}_t, \bar{\vect{p}}_t) \gets (\vect{X}, \bar{\vect{p}})$
\State $\vect{p}_t = \bar{\vect{p}}_t - \nu \sqrt{w_t}$ for $w_t$ defined in \eqref{w_t_def}
\State Offer the allocation $\vect{X}_t$ together with prices $\vect{p}_t$
\State Observe $R^{t}_{ui}$ for all $(u, i) \in \mathcal{A}_{t}$
\EndFor
\end{algorithmic}
\label{alg_mf_lr}
\end{algorithm}

In order to efficiently solve the least squares problem given in line~\ref{alg:line:least_squares} of Algorithm \ref{alg_low} as well as the OFU problem \eqref{low_oful} in large scales, we take advantage of the matrix factorization model. As a result, we factorize $\vect{\Theta} = \vect{F} \vect{\Phi}^\textrm{T}$ where $\vect{F} \in \mathbb{R}^{N \times R}$ and $\vect{\Phi} \in \mathbb{R}^{M \times R}$, and solve the problems by optimizing over $\vect{F}$ and $\vect{\Phi}$ rather than directly optimizing over $\vect{\Theta}$. Even if the least squares problem is not convex in the joint variable ($\vect{F}$, $\vect{\Phi}$), it is convex in $\vect{F}$ for fixed $\vect{\Phi}$ and it is convex in $\vect{\Phi}$ for fixed $\vect{F}$. Therefore, an alternating minimization algorithm becomes a feasible choice to find a reasonable solution for the least squares problem even in the case with unknown $\vect{\Phi}$ matrix. Similarly, an alternating minimization approach is also useful to solve the problem \eqref{low_oful}. We can fix an allocation $\vect{X}$ and minimize over $\vect{F}$ and $\vect{\Phi}$. Then, for fixed $\vect{F}$ and $\vect{\Phi}$, the allocation $\vect{X}$ and prices $\bar{\vect{p}}$ are determined through the linear program described in the section \ref{sect_opt_allocations}. Note that converged $\widehat{\vect{\Theta}}_t$ and $\vect{X}_t$ are not necessarily the optimum solution in the case with unknown $\vect{\Phi}$ matrix, since the optimization problems are not convex. However, the alternating optimization algorithm guarantees that the objective value only improves in each iteration of the alternating optimization.

\section{DEFINITIONS}

We start with definitions that we will use throughout the proofs.

\begin{definition}
For a symmetric positive definite matrix $\vect{W} \in \mathbb{R}^{d \times d}$, we define
\begin{itemize}
    \item $\vect{W}$-inner product of two vectors $\vect{x}, \vect{y} \in \mathbb{R}^{d}$ as $\langle \vect{x}, \vect{y} \rangle_{\vect{W}} := \langle \vect{W}  \vect{x}, \vect{y} \rangle$
    \item $\vect{W}$-norm of a vector $\vect{x} \in \mathbb{R}^{d}$ as $\|\vect{x}\|_{\vect{W}} := \sqrt{\langle \vect{x}, \vect{x} \rangle_{\vect{W}}}$.
\end{itemize}
\end{definition}

\begin{definition}
$\vect{A}_t \in \mathds{R}^{NM \times NM}$ is a diagonal matrix with $(N(i-1)+u)$th diagonal entry equal to $n^t_{ui} + \gamma$. Recall that $n^t_{ui} = \sum_{\tau=1}^{t-1} \mathds{1} \{(u,i) \in \mathcal{A}_t\}$ denotes the number of times pair $(u,i)$ has been sampled before time $t$ (excluding time $t$).
\end{definition} 

Then, the regularized empirical $L_2$-norm of a matrix $\vect{Z} \in \mathbb{R}^{N \times M}$ can be written as 
\begin{equation}
    \| \vect{Z} \|_{E^t_2} = \| \mathrm{vec} (\vect{Z}) \|_{\vect{A}_t}
\end{equation}

For the ease of exposition in the following sections, we will use the shorthand $ \| \vect{Z} \|_{\vect{A}_t} = \| \mathrm{vec} (\vect{Z}) \|_{\vect{A}_t}$.

\begin{definition}
Define $\vect{E}_{ui} := \vect{e}_u \vect{e}_i^\textrm{T}$ to be the indicator matrix for $(u,i)$ and $\vect{e}_{ui} := \mathrm{vec}(\vect{E}_{ui})$.    
\end{definition}

\begin{definition}
Define the "confidence width" of the allocations as
\begin{equation}
    w_t := \| \vect{X}_t \|_{\vect{A}_t^{-1}} \qquad \text{and} \qquad w^t_{ui} := \| \vect{E}_{ui} \|_{\vect{A}_t^{-1}}
\label{def_widths}
\end{equation} 
\end{definition}

\begin{definition}
Define the regularized empirical $L_2$-norm $\| \cdot \|_{E_{2}^t}$ as 
\begin{equation*}
    \| \vect{\Delta} \|_{E_{2}^t}^2 := \sum_{\tau=1}^{t-1} \sum_{(u, i) \in \mathcal{A}_\tau} \langle \vect{\Delta}, \vect{E}_{ui} \rangle^2 + \gamma \| \vect{\Delta} \|_{\mathrm{F}}^2 =  \sum_{u \in \mathcal{N}} \sum_{i\in \mathcal{I}} (n^t_{ui} + \gamma) (\Delta_{ui})^2
\end{equation*}
For future reference, we also define the (non-regularized) empirical $L_2$-norm $\| \cdot \|_{\widetilde{E}_{2}^t}$ by 
\begin{equation*}
    \| \vect{\Delta} \|^2_{\widetilde{E}_{2}^t} := \sum_{\tau=1}^{t-1} \sum_{(u, i) \in \mathcal{A}_\tau} \langle \vect{\Delta}, \vect{E}_{ui} \rangle^2 =  \sum_{i=1}^{d} (n^{t}_{ui}) (\Delta_{ui})^2
\end{equation*}

\end{definition}

Note that the regularized empirical $L_2$-norm is related to (non-regularized) empirical $L_2$-norm as
\begin{equation*}
    \| \vect{\Delta} \|^2_{E^t_2} = \| \vect{\Delta} \|^2_{\widetilde{E}^t_2} + \gamma \| \vect{\Delta} \|^2_{\mathrm{F}}
\end{equation*}

\begin{definition}
Define the social welfare regret $r_t$ at time $t$ as 
\begin{equation*}
    r_t := \langle \vect{X}_t^*, \vect{\Theta}^* \rangle - \langle \vect{X}_t, \vect{\Theta}^* \rangle
\end{equation*} 
\end{definition}

\begin{definition}
The OFU estimate of the allocation and the parameters are 
\begin{equation*}
(\vect{X}_t, \vect{\Theta}_t) \in \argmax_{(\vect{X}, \vect{\Theta}) \in \mathcal{X}_t \times \mathcal{C}_t} \; \langle \vect{X}, \vect{\Theta} \rangle
\end{equation*} 
The entries of $\vect{X}_t$ are denoted by $X^t_{ui}$ and the entries of $\vect{\Theta}_t$ are denoted by $\Theta^t_{ui}$.
\end{definition}

\begin{definition}
    We define  $\vect{\Delta}_t = \vect{\Theta}_{t} - \vect{\Theta}^{*}$ with entries $\Delta^t_{ui} = \Theta_{ui}^t - \Theta^*_{ui}$
\end{definition}

\section{PROOFS FOR LOW-RANK ALLOCATION AND PRICING}

\subsection{Construction of Confidence Sets}
\label{pf_conf_sets}

\begin{lemma} For any $\delta > 0$ and $\vect{\Theta} \in \mathbb{R}^{d}$,
\begin{equation}
    \mathds{P} \left( L_{2,t}(\vect{\Theta}) \geq L_{2,t}(\vect{\Theta}^*) + \frac{1}{2} \|\vect{\Theta}^* - \vect{\Theta} \|^2_{\widetilde{E}^t_2} - 4 \eta^2 \log(1/\delta) \quad ,\forall t \in \mathbb{N} \right) \geq 1 - \delta
\end{equation}

\label{lemma_lower_bound}
\end{lemma}

\begin{proof}

Let $\mathcal{H}_{t-1}$ be the $\sigma$-algebra generated by $(H_t, \mathcal{A}_t)$ and let $\mathcal{H}_0 = \sigma(\emptyset, \Omega)$. Then, define $\epsilon^{t}_{ui} := R^{t}_{ui} - \Theta^*_{ui}$ for all $t \in \mathds{N}$ and $(u, i) \in \mathcal{A}_t$. By previous assumptions, $\mathds{E} [\epsilon^{t}_{ui} | \mathcal{H}_{t-1}] = 0$ and $\mathds{E} [\exp (\lambda \epsilon^{t}_{ui}) | \mathcal{H}_{t-1} ] \leq \exp \left(\frac{\lambda^2 \eta^2}{2} \right)$ for all $t$. 

Define $Z^{t}_{ui} := (R^{t}_{ui} - \Theta^*_{ui})^2 - (R^{t}_{ui} -  \Theta_{ui})^2$. Then, we have
\begin{align*}
Z^{t}_{ui} &= - (\Theta_{ui} - \Theta^*_{ui})^2 + 2 \epsilon^{t}_{ui} (\Theta_{ui} - \Theta^*_{ui})
\end{align*}

Therefore, the conditional mean and conditional cumulant generating function satisfy
\begin{align*}
    \mu^{t}_{ui} &:= \mathds{E}[Z^{t}_{ui} | \mathcal{H}_{t-1}] = - (\Theta_{ui} - \Theta^*_{ui})^2\\
    \psi^t_{ui}(\lambda) &:= \log \mathds{E}[ \exp (\lambda [Z^{t}_{ui} - \mu^{t}_{ui}]) | \mathcal{H}_{t-1}]\\
    &= \log \mathds{E}[ \exp (2 \lambda (\Theta_{ui} - \Theta^*_{ui}) \epsilon^{t}_{ui} ) | \mathcal{H}_{t-1}]\\
    &\leq \frac{(2 \lambda (\Theta_{ui} - \Theta^*_{ui}))^2 \eta^2}{2}
\end{align*}

Then, define $Z^{t} := \sum_{(u,i) \in \mathcal{A}_t} Z^{t}_{ui}$. Since each $Z^{t}_{ui}$ term is conditionally independent, the conditional mean and conditional cumulant generating functions of $Z^{t}$ satisfy
\begin{align*}
    \mu^{t} &:= \mathds{E}[Z^{t} | \mathcal{H}_{t-1}] = -  \sum_{(u,i) \in \mathcal{A}_t} (\Theta_{ui} - \Theta^*_{ui})^2\\
    \psi^t(\lambda) &:= \sum_{(u,i) \in \mathcal{A}_t} \psi^t_{ui}(\lambda)\\
    &\leq \sum_{(u,i) \in \mathcal{A}_t} \frac{(2 \lambda (\Theta_{ui} - \Theta^*_{ui}))^2 \eta^2}{2}
\end{align*}

Applying Lemma \ref{exp_martingale} to $(Z^t)_{t \in \mathds{N}}$ shows that for all $x \geq 0$ and $\lambda \geq 0$,
\begin{equation*}
\mathds{P} \left(\sum_{\tau=1}^{t-1} Z^{\tau} \leq \frac{x}{ \lambda} + \sum_{\tau=1}^{t-1} \sum_{(u, i) \in \mathcal{A}_\tau} (\Theta_{ui} - \Theta^*_{ui})^2 (2 \lambda \eta^2 - 1) \quad ,\forall t \in \mathds{N} \right) \geq 1 - e^{-x}
\end{equation*}

Note that we have $\sum_{\tau=1}^{t-1} Z^{\tau} = L_{2,t}(\vect{\Theta}^*) - L_{2,t}(\vect{\Theta})$, and $\sum_{\tau=1}^{t-1} \sum_{(u, i) \in \mathcal{A}_\tau} (\Theta_{ui} - \Theta^*_{ui})^2 = \|\vect{\Theta}^* - \vect{\Theta} \|^2_{\widetilde{E}^t_2}$. 

Then, choosing $\lambda = \frac{1}{4 \eta^2}$ and $x = \log \frac{1}{\delta}$ gives
\begin{equation*}
    \mathds{P} \left( L_{2,t}(\vect{\Theta}) \geq L_{2,t}(\vect{\Theta}^*) + \frac{1}{2} \|\vect{\Theta}^* - \vect{\Theta} \|^2_{\widetilde{E}^t_2} - 4 \eta^2 \log(1/\delta) \quad ,\forall t \in \mathbb{N} \right) \geq 1 - \delta
\end{equation*}

\end{proof}

\begin{lemma}
If $\vect{\Theta}, \vect{\Theta}^{\alpha} \in \mathcal{L}$ satisfy $\|\vect{\Theta} - \vect{\Theta}^{\alpha}\|_\mathrm{F} \leq \alpha$, then with probability at least $1 - \delta$,
\begin{equation}
    \left |  \frac{1}{2} \|\vect{\Theta}^* - \vect{\Theta}^{\alpha} \|^2_{\widetilde{E}^t_2} - \frac{1}{2} \|\vect{\Theta}^* - \vect{\Theta} \|^2_{\widetilde{E}^t_2} + L_{2,t}(\vect{\Theta}) - L_{2,t}(\vect{\Theta}^{\alpha}) \right | \leq \alpha t \sqrt{NM} \left [ 8 + 2 \sqrt{2 \eta^2 \log(4 NM t^2/\delta)} \right]
\label{eqn_of_discr_lemma}
\end{equation}
\label{discr_lemma}
\end{lemma}

\begin{proof}

Let $\vect{E}_{ui} = \vect{e}_u \vect{e}_i^\textrm{T}$ be the indicator matrix for $(u,i)$. Since any two $\vect{\Theta}, \vect{\Theta}^{\alpha} \in \mathcal{L}$ satisfy $\|\vect{\Theta} - \vect{\Theta}^{\alpha}\|_\mathrm{F} \leq \sqrt{N M}$, it is enough to consider $\alpha \leq  \sqrt{NM}$. We find 
\begin{align*}
    \sum_{ u = 1 }^{N} \sum_{ i = 1 }^{M} | \langle \vect{\Theta}, \vect{E}_{ui} \rangle^2 - \langle \vect{\Theta}^{\alpha}, \vect{E}_{ui} \rangle^2 | &\leq \max_{ \|\vect{\Delta}\|_\mathrm{F} \leq \alpha } \left \{ \sum_{ u = 1 }^{N} \sum_{ i = 1 }^{M} \left | \Theta_{ui}^2 - (\Theta_{ui}+\Delta_{ui})^2  \right | \right \}\\
    &= \max_{ \|\vect{\Delta}\|_\mathrm{F} \leq \alpha } \left \{ \sum_{ u = 1 }^{N} \sum_{ i = 1 }^{M} \left | 2 \Theta_{ui} \Delta_{ui} + \Delta_{ui}^2  \right | \right \}\\
    &\leq \max_{ \|\vect{\Delta}\|_\mathrm{F} \leq \alpha } \left \{ 2 \sum_{ u = 1 }^{N} \sum_{ i = 1 }^{M} \left | \Theta_{ui} \Delta_{ui} \right | + \sum_{ u = 1 }^{N} \sum_{ i = 1 }^{M} \Delta_{ui}^2 \right \}\\
    &\leq \max_{ \|\vect{\Delta}\|_\mathrm{F} \leq \alpha }  \left \{ 2 \|\vect{\Delta}\|_{1, 1} + \|\vect{\Delta}\|_\textrm{F}^2 \right \}\\
    &\leq 2  \alpha \sqrt{ N M} + \alpha^2
\end{align*}
Therefore, it implies
\begin{align*}
     \sum_{ u = 1 }^{N} \sum_{ i = 1 }^{M}  | \langle \vect{\Theta} - \vect{\Theta}^*, \vect{E}_{ui} \rangle^2 - \langle \vect{\Theta}^{\alpha} - \vect{\Theta}^*, \vect{E}_{ui} \rangle^2 | &= \sum_{ u = 1 }^{N} \sum_{ i = 1 }^{M} \left| \langle \vect{\Theta}, \vect{E}_{ui} \rangle^2 - \langle \vect{\Theta}^{\alpha}, \vect{E}_{ui} \rangle^2 + 2 \langle \vect{\Theta}^*, \vect{E}_{ui} \rangle \langle \vect{\Theta}^{\alpha} - \vect{\Theta}, \vect{E}_{ui} \rangle \right| \\
     &\leq 2  \alpha \sqrt{ N M} + \alpha^2 + 2 \|\vect{\Theta} - \vect{\Theta}^{\alpha}\|_{1, 1} \\
     &\leq 4 \alpha \sqrt{ N M} + \alpha^2
\end{align*}

Similarly, for any $t$, we have
\begin{align*}
     \sum_{ u = 1 }^{N} \sum_{ i = 1 }^{M} | \left( R^t_{ui} - \langle \vect{\Theta}, \vect{E}_{ui} \rangle \right)^2 -\left( R^t_{ui} - \langle \vect{\Theta}^{\alpha}, \vect{E}_{ui} \rangle \right)^2 | &= \sum_{ u = 1 }^{N} \sum_{ i = 1 }^{M} \left| 2 R^t_{ui} \langle \vect{\Theta}^{\alpha} - \vect{\Theta}, \vect{E}_{ui} \rangle + \langle \vect{\Theta}, \vect{E}_{ui} \rangle^2 - \langle \vect{\Theta}^{\alpha}, \vect{E}_{ui} \rangle^2 \right| \\
      &\leq 2 \sum_{ u = 1 }^{N} \sum_{ i = 1 }^{M} \left| R^t_{ui} \right | \left| \langle \vect{\Theta}^{\alpha} - \vect{\Theta}, \vect{E}_{ui} \rangle \right| + 2  \alpha \sqrt{ N M} + \alpha^2\\
     &\leq 2 \| \vect{\Theta}^{\alpha} - \vect{\Theta}\|_\mathrm{F} \left( \sum_{ u = 1 }^{N} \sum_{ i = 1 }^{M} | R^t_{ui} |^2 \right)^{1/2} + 2  \alpha \sqrt{ N M} + \alpha^2\\
     &\leq 2 \alpha \left( \sum_{ u = 1 }^{N} \sum_{ i = 1 }^{M} | R^t_{ui} |^2 \right)^{1/2} + 2  \alpha \sqrt{ N M} + \alpha^2
\end{align*}

Summing over $t$ and noting that $\mathcal{A}_t \subseteq [N] \times [M]$, the left hand side of \eqref{eqn_of_discr_lemma} is bounded by
\begin{equation*}
    \sum_{\tau = 1}^{t-1} \left( \frac{1}{2} \left[ 4 \alpha \sqrt{ N M} + \alpha^2 \right] + 2 \alpha \left( \sum_{ u = 1 }^{N} \sum_{ i = 1 }^{M} | R^t_{ui} |^2 \right)^{1/2} + 2  \alpha \sqrt{ N M} + \alpha^2 \right) \leq \alpha \sum_{\tau = 1}^{t-1} \left(6 \sqrt{NM} + 2 \left( \sum_{ u = 1 }^{N} \sum_{ i = 1 }^{M} | R^t_{ui} |^2 \right)^{1/2} \right)
\end{equation*}

Because $\epsilon^{\tau}_{ui}$ is $\eta$-sub-Gaussian, $\mathds{P} \left( |\epsilon^{\tau}_{ui}| > \sqrt{2 \eta^2 \log(2/\delta) }\right) \leq \delta$. By a union bound, $\mathds{P} \left( \exists \tau \in \mathds{N}, u \in [N], i \in [M] \text{ s.t. } |\epsilon^{\tau}_{ui}| > \sqrt{2 \eta^2 \log(4 N M \tau^2/\delta) }\right) \leq \frac{\delta NM}{2} \sum_{\tau = 1}^{\infty} \frac{1}{N M \tau^2} \leq \delta$. Since $| R^{\tau}_{ui} | \leq 1 + | \epsilon^{\tau}_{ui} |$, we have $| R^{\tau}_{ui} | \leq 1 + \sqrt{2 \eta^2 \log(4 NM \tau^2/\delta)}$  with probability at least $1 - \delta$. Consequently, the bound for the discretization error becomes
\begin{equation*}
\alpha t \sqrt{NM} \left [ 8 + 2 \sqrt{2 \eta^2 \log(4 NM t^2/\delta)} \right]
\end{equation*}

\end{proof}

\begin{lemma} For any $\delta > 0$, $\alpha > 0$ and $\gamma > 0$, if
\begin{equation}
    \mathcal{Q}_t = \{ \vect{\Theta} \in \mathcal{F} : \|\vect{\Theta} - \widehat{\vect{\Theta}}_t \|_{E^t_2} \leq \sqrt{ \beta_t^*(\delta, \alpha, \gamma)}\}
\end{equation}
for all $t \in \mathbb{N}$, then
\begin{equation}
    \mathds{P} \left( \vect{\Theta}^* \in \mathcal{Q}_t \quad ,\forall t \in \mathbb{N} \right) \geq 1 - 2\delta
\end{equation}
\label{lemma_confidence}
\end{lemma}

\begin{proof}

Let $\mathcal{F}^{\alpha} \subset \mathcal{F}$ be an $\alpha$-cover of $\mathcal{F}$ in the Frobenius norm so that for any $\vect{\Theta} \in \mathcal{F}$, there exists $\vect{\Theta}^{\alpha} \in \mathcal{F}^{\alpha}$ such that $\|\vect{\Theta} - \vect{\Theta}^{\alpha}\|_\mathrm{F} \leq \alpha$. By a union bound applied to Lemma \ref{lemma_lower_bound}, with probability at least $1 - \delta$,
\begin{equation*}
    L_{2,t}(\vect{\Theta}^{\alpha}) - L_{2,t}(\vect{\Theta}^*) \geq \frac{1}{2} \|\vect{\Theta}^* - \vect{\Theta}^{\alpha} \|^2_{\widetilde{E}^t_2} - 4 \eta^2 \log(|\mathcal{F}^{\alpha}|/\delta) \quad ,\forall \vect{\Theta}^{\alpha} \in \mathcal{F}^{\alpha}, t \in \mathbb{N}
\end{equation*}

Therefore, with probability at least $1 - \delta$, for all $\vect{\Theta} \in \mathcal{F}, t \in \mathbb{N}$,
\begin{align*}
    L_{2,t}(\vect{\Theta}) - L_{2,t}(\vect{\Theta}^*) \geq & \frac{1}{2} \|\vect{\Theta}^* - \vect{\Theta} \|^2_{\widetilde{E}^t_2} - 4 \eta^2 \log(|\mathcal{F}^{\alpha}|/\delta) \\
    &+ \min_{\vect{\Theta}^{\alpha} \in \mathcal{F}^{\alpha}} \left \{  \frac{1}{2} \|\vect{\Theta}^* - \vect{\Theta}^{\alpha} \|^2_{\widetilde{E}^t_2} - \frac{1}{2} \|\vect{\Theta}^* - \vect{\Theta} \|^2_{\widetilde{E}^t_2} + L_{2,t}(\vect{\Theta}) - L_{2,t}(\vect{\Theta}^{\alpha}) \right \}
\end{align*}

By Lemma \ref{discr_lemma}, with probability at least $1 - 2 \delta$,
\begin{equation*}
    L_{2,t}(\vect{\Theta}) - L_{2,t}(\vect{\Theta}^*) \geq \frac{1}{2} \|\vect{\Theta}^* - \vect{\Theta} \|^2_{\widetilde{E}^t_2} - D_t
\end{equation*}
where $D_t := 4 \eta^2 \log(|\mathcal{F}^{\alpha}|/\delta) + \alpha t \sqrt{NM} \left [ 8 + 2 \sqrt{2 \eta^2 \log(4 NM t^2/\delta)} \right]$.

Adding the regularization terms to both sides, we obtain
\begin{equation*}
    L_{2,t}(\vect{\Theta}) + \gamma \|\vect{\Theta} - \overline{\vect{\Theta}} \|_\mathrm{F}^2 - L_{2,t}(\vect{\Theta}^*) - \gamma \|\vect{\Theta}^* - \overline{\vect{\Theta}} \|_\mathrm{F}^2 \geq \frac{1}{2} \|\vect{\Theta}^* - \vect{\Theta} \|^2_{\widetilde{E}^t_2} + \gamma \|\vect{\Theta} - \overline{\vect{\Theta}} \|_\mathrm{F}^2 - D_t - \gamma  \|\vect{\Theta}^* - \overline{\vect{\Theta}} \|_\mathrm{F}^2 
\end{equation*}

Note the definition of the least square estimate $\widehat{\vect{\Theta}}_t = \argmin_{\vect{\Theta} \in \mathcal{F}} \left \{ L_{2,t}(\vect{\Theta}) + \gamma \|\vect{\Theta} - \overline{\vect{\Theta}} \|_\mathrm{F}^2 \right \}$. By letting $\vect{\Theta} = \widehat{\vect{\Theta}}_t$, the left hand side becomes non-positive, and hence 
\begin{equation*}
\frac{1}{2} \|\vect{\Theta}^* - \widehat{\vect{\Theta}}_t \|^2_{\widetilde{E}^t_2} \leq D_t + \gamma \left( \|\vect{\Theta}^* - \overline{\vect{\Theta}} \|_\mathrm{F}^2 - \|\widehat{\vect{\Theta}}_t - \overline{\vect{\Theta}} \|_\mathrm{F}^2 \right)
\end{equation*}

Then, 
\begin{equation*}
\frac{1}{2} \|\vect{\Theta}^* - \widehat{\vect{\Theta}}_t \|^2_{\widetilde{E}^t_2} + \gamma \left( \|\widehat{\vect{\Theta}}_t - \overline{\vect{\Theta}} \|_\mathrm{F}^2 + \|\vect{\Theta}^* - \overline{\vect{\Theta}} \|_\mathrm{F}^2  \right) \leq D_t + 2 \gamma \|\vect{\Theta}^* - \overline{\vect{\Theta}} \|_\mathrm{F}^2
\end{equation*}

By triangle inequality we have $ \|\widehat{\vect{\Theta}}_t - \overline{\vect{\Theta}} \|_\mathrm{F} + \|\vect{\Theta}^* - \overline{\vect{\Theta}} \|_\mathrm{F}  \geq \|\vect{\Theta}^* - \widehat{\vect{\Theta}}_t \|_\mathrm{F}$. Taking squares on both sides, we obtain $ \|\widehat{\vect{\Theta}}_t - \overline{\vect{\Theta}} \|_\mathrm{F}^2 + \|\vect{\Theta}^* - \overline{\vect{\Theta}} \|_\mathrm{F}^2  \geq \frac{1}{2} \|\vect{\Theta}^* - \widehat{\vect{\Theta}}_t \|_\mathrm{F}^2$. Then, noting that $\| \vect{\Delta} \|^2_{2, E_t^2} = \| \vect{\Delta} \|^2_{\widetilde{E}^t_2} + \gamma \| \vect{\Delta} \|^2_\mathrm{F}$, we have
\begin{equation*}
\frac{1}{2} \|\vect{\Theta}^* - \widehat{\vect{\Theta}}_t \|^2_{E^t_2} \leq D_t + 2 \gamma \|\vect{\Theta}^* - \overline{\vect{\Theta}} \|_\mathrm{F}^2
\end{equation*}

Lastly, using the inequality $\|\vect{\Theta}^* - \overline{\vect{\Theta}} \|_\mathrm{F}^2 \leq G^2$,
\begin{equation*}
\|\vect{\Theta}^* - \widehat{\vect{\Theta}}_t \|^2_{E^t_2} \leq 8 \eta^2 \log(|\mathcal{F}^{\alpha}|/\delta) + 2 \alpha t \sqrt{NM} \left [ 8 + 2 \sqrt{2 \eta^2 \log(4 NM t^2/\delta)} \right] + 4 \gamma G^2
\end{equation*}

Taking the infimum over the size of $\alpha$-covers, we obtain the final result.

\end{proof}

\subsection{Regret Bounds}
\label{pf_regrets}

Throughout this section we will use the shorthand $\beta_t =  \beta_t^*(\delta, \alpha, \gamma)$ for the parameter defined as
\begin{equation}
    \beta_t^*(\delta, \alpha, \gamma) := 8 \eta^2 \log \left(\mathcal{N}(\mathcal{F}, \alpha, \| \cdot \|_{2}) / \delta \right) + 2 \alpha t \sqrt{N M} \left(8 + \sqrt{8 \eta^2 \log(4 N M t^2  / \delta)} \right) + 4 \gamma G^2
\end{equation}

\begin{lemma}
For any $\vect{X}_t \in \mathcal{X}_t$ and $\vect{\Theta} \in \mathcal{Q}_t$, we have
\begin{equation}
    |\langle \vect{\Theta} - \widehat{\vect{\Theta}}^t, \vect{X} \rangle| \leq w_t \sqrt{\beta_t}
\end{equation}
\label{lemma_conf_width}
\end{lemma}

\begin{proof} Let $\vect{\Delta} = \vect{\Theta} - \widehat{\vect{\Theta}}^t$. Then,
\begin{align*}
    |\langle \vect{\Delta} , \vect{X}_t \rangle| &= |\textrm{vec}(\vect{\Delta})^\textrm{T} \textrm{vec}(\vect{X}_t) |\\
    &= |\textrm{vec}(\vect{\Delta})^\textrm{T} \vect{A}_t^{1/2} \vect{A}_t^{-1/2} \textrm{vec}(\vect{X}_t)|\\
    &= |(\vect{A}_t^{1/2}\textrm{vec}(\vect{\Delta}))^\textrm{T} \vect{A}_t^{-1/2} \textrm{vec}(\vect{X}_t)|\\
    &\leq \| \vect{A}_t^{1/2}\textrm{vec}(\vect{\Delta}) \| \|\vect{A}_t^{-1/2} \textrm{vec}(\vect{X}_t)\|\\
    &= \| \textrm{vec}(\vect{\Delta}) \|_{\vect{A}_t} \| \textrm{vec}(\vect{X}_t) \|_{\vect{A}_t^{-1}}\\
    &= \| \vect{\Delta} \|_{\vect{A}_t} \| \vect{X}_t \|_{\vect{A}_t^{-1}}\\
    &\leq w_t \sqrt{\beta_t} 
\end{align*}
\end{proof}

\begin{lemma}
For any $t \in \mathbb{N}$, we have the identity
\begin{equation*}
    w_t^2 = \sum_{(u, i) \in \mathcal{A}_t} (w^{t}_{ui})^2
\end{equation*}
\label{lemma_width_identity}
\end{lemma}
\begin{proof}
\begin{align*}
    w_t^2 &= \langle \mathrm{vec}(\vect{X}_t), \mathrm{vec}(\vect{X}_t) \rangle_{\vect{A}_t^{-1}}\\
    &= \left\langle \vect{A}_t^{-1} \sum_{(u,i) \in \mathcal{A}_t} \vect{e}_{ui}, \sum_{(u,i) \in \mathcal{A}_t} \vect{e}_{ui} \right\rangle \\
    &= \sum_{(u,i) \in \mathcal{A}_t} \sum_{(r, j) \in \mathcal{A}_t} \left\langle \vect{A}_t^{-1} \vect{e}_{ui}, \vect{e}_{rj} \right\rangle\\
    &= \sum_{(u,i) \in \mathcal{A}_t} \left\langle \vect{A}_t^{-1}\vect{e}_{ui}, \vect{e}_{ui} \right\rangle\\
    &= \sum_{(u,i) \in \mathcal{A}_t} (w^{t}_{ui})^2
\end{align*}
where the penultimate step follows because $\left\langle \vect{A}_t^{-1} \vect{e}_{ui}, \vect{e}_{rj} \right\rangle = 0$ for $(u, i) \neq (r, j)$.
\end{proof}

\begin{lemma}
 If $\vect{\Theta}^* \in \mathcal{Q}_t$, then
\begin{equation*}
    r_t \leq \langle \vect{X}_t, \vect{\Theta}_t - \vect{\Theta}^* \rangle \leq 2 w_t \sqrt{\beta_t}
\end{equation*}
\label{lemma_regret_ub}
\end{lemma}

\begin{proof}
By the choice of $(\vect{X}_t, \vect{\Theta}_t)$, we have
\begin{align*}
    \langle \vect{X}_t, \vect{\Theta}_t \rangle = \max_{(\vect{X}, \vect{\Theta}) \in \mathcal{X}_t \times \mathcal{Q}_t} \langle \vect{X}, \vect{\Theta} \rangle \geq \langle \vect{X}_t^*, \vect{\Theta}^* \rangle
\end{align*}
where the inequality uses $\vect{\Theta}^* \in \mathcal{Q}_t$. Hence,
\begin{align*}
    r_t &= \langle \vect{X}_t^*, \vect{\Theta}^* \rangle - \langle \vect{X}_t, \vect{\Theta}^* \rangle\\
    &\leq \langle \vect{X}_t, \vect{\Theta}_t - \vect{\Theta}^* \rangle\\
    &= \langle \vect{X}_t, \vect{\Theta}_t - \widehat{\vect{\Theta}}^{t} \rangle + \langle \vect{X}_t, \widehat{\vect{\Theta}}^t - \vect{\Theta}^* \rangle\\
    &\leq  2 w_t \sqrt{\beta_t}
\end{align*}

where the last step follows from Lemma $\ref{lemma_conf_width}$.
\end{proof}

Next, we show that the confidence widths do not grow too fast.

\begin{lemma}
For every t,
\begin{equation}
    \log \det \vect{A}_{t+1} = N M \log \gamma + \sum_{\tau = 1}^{t} \sum_{(u, i) \in \mathcal{A}_\tau} \log(1 + (w^{\tau}_{ui})^2) 
\end{equation}
\label{lemma_logdet_multiply}
\end{lemma}

\begin{proof} By the definition of $\vect{A}_t$, we have
\begin{align*}
    \det \vect{A}_{t+1} &= \det \left( \vect{A}_{t} + \sum_{(u, i) \in \mathcal{A}_t} \vect{e}_{ui} \vect{e}_{ui}^\textrm{T} \right)\\
    &= \det \left( \vect{A}_{t}^{1/2} \left( \vect{I} + \vect{A}_{t}^{-1/2}\left( \sum_{(u, i) \in \mathcal{A}_t} \vect{e}_{ui} \vect{e}_{ui}^\textrm{T} \right) \vect{A}_{t}^{-1/2} \right) \vect{A}_{t}^{1/2} \right) \\
    &= \det (\vect{A}_{t}) \det\left( \vect{I} + \sum_{(u, i) \in \mathcal{A}_t} \vect{A}_{t}^{-1/2}  \vect{e}_{ui} \vect{e}_{ui}^\textrm{T} \vect{A}_{t}^{-1/2}  \right)
\end{align*}

Each $\vect{A}_{t}^{-1/2}  \vect{e}_{ui} \vect{e}_{ui}^\textrm{T} \vect{A}_{t}^{-1/2}$ term has zeros everywhere except one entry on the diagonal and that non-zero entry is equal to $(w^{t}_{ui})^2$. Furthermore, the location of the non-zero entry is different in for each term. Hence,
\begin{equation*}
    \det\left( \vect{I} + \sum_{(u, i) \in \mathcal{A}_t} \vect{A}_{t}^{-1/2}  \vect{e}_{ui} \vect{e}_{ui}^\textrm{T} \vect{A}_{t}^{-1/2}  \right) = \prod_{(u, i) \in \mathcal{A}_t} (1 + (w^{t}_{ui})^2)
\end{equation*}

Therefore, we have
\begin{equation*}
    \log \det \vect{A}_{t+1} = \log\det \vect{A}_{t} + \sum_{(u, i) \in \mathcal{A}_t} \log(1 + (w^{t}_{ui})^2)
\end{equation*}
Since $\vect{A}_1 = \gamma \vect{I}$, we have $\log \det \vect{A}_1 = NM \log \gamma$ and the result follows by induction.
\end{proof}

\begin{lemma}
For all $t \in \mathds{N}$, $\log \det \vect{A}_t \leq N M \log (t + \gamma - 1)$.
\label{lemma_logdet_upper}
\end{lemma}
\begin{proof}
Noting that $\vect{A}_t$ is a diagonal matrix with diagonals $(n^{t}_{ui} + \gamma)$,
\begin{align*}
    \tr \vect{A}_t &= N M \gamma + \sum_{u = 1}^{N} \sum_{i = 1}^{M} n^{t}_{ui} \\
    &\leq N M \gamma + N M (t-1) \\
    &= N M (t + \gamma - 1)
\end{align*}
Now, recall that $\tr \vect{A}_t$ equals the sum of the eigenvalues of $\vect{A}_t$. On the other hand, $\det(\vect{A}_t)$ equals the product of the eigenvalues. Since $\vect{A}_t$ is positive definite, its eigenvalues are all positive. Subject to these constraints, $\det(\vect{A}_t)$ is maximized when all the eigenvalues are equal; the desired bound follows. 
\end{proof}

\begin{lemma}
Let $\gamma \geq 1$. Then, for all t, we have
\begin{equation*}
    \sum_{\tau = 1}^{t} w_{\tau}^2 \leq 2 N M \log \left(1 + \frac{t}{\gamma} \right)
\end{equation*}
\label{lemma_width_sum_ub}
\end{lemma}

\begin{proof}
Note that $0 \leq (w^{\tau}_{ui})^2 \leq 1$, if $\gamma \geq 1$. Using the inequality $y \leq 2 \log(1 + y)$ for $0 \leq y \leq 1$, we have
\begin{align*}
    \sum_{\tau = 1}^{t} w_{\tau}^2 &= \sum_{\tau = 1}^{t} \sum_{(u, i) \in \mathcal{A}_\tau} (w^{\tau}_{ui})^2 \\
    &\leq 2 \sum_{\tau = 1}^{t} \sum_{(u, i) \in \mathcal{A}_\tau} \log (1 + (w^{\tau}_{ui})^2)\\
    &= 2 \log \det \vect{A}_{t+1} - 2 N M \log \gamma\\
    &\leq 2 N M \log \left(1 + \frac{t}{\gamma} \right)
\end{align*}
where the last two lines follow from Lemmas \ref{lemma_logdet_multiply} and \ref{lemma_logdet_upper} respectively.
\end{proof}

\begin{theorem}
If $\gamma \geq 1$ and $\vect{\Theta}^* \in \mathcal{Q}_t$ for all $t \leq T$, then the T period social welfare regret is bounded by
\begin{equation*}
    \mathcal{R}^{SW}(T, \pi) \leq \sqrt{ 8 N M \beta_T^* (\delta, \alpha, \gamma) T \log \left(1 + \frac{T}{\gamma} \right) }
\end{equation*}
\label{thm_regret}
\end{theorem}

\begin{proof}
Recall the definition of instantaneous regret $r_t = \langle \vect{X}_t^*, \vect{\Theta}^* \rangle - \langle \vect{X}_t, \vect{\Theta}^* \rangle$. Assuming that $\vect{\Theta}^* \in \mathcal{Q}_t$ for all $t \leq T$, we have
\begin{align*}
    \mathcal{R}^{SW}(T, \pi) &= \sum_{t = 1}^{T} r_t \\
    &\leq \left( T \sum_{t = 1}^{T} r_t^2 \right)^{1/2}\\
    &\leq \left( 4 \beta_T T \sum_{t = 1}^{T} w_t^2 \right)^{1/2}
\end{align*}
where the last step follows from Lemma \ref{lemma_regret_ub}. Then, using Lemma \ref{lemma_width_sum_ub}, we prove the result.
\end{proof}

In order to prove our instability bounds, we start with the following lemma.

\begin{lemma}
    If $\vect{\Theta}^* \in \mathcal{Q}_t$, we have
    \begin{equation*}
        \sum_{u=1}^{N} \left\{ \max_{\vect{x}_u \in \mathcal{X}_u^t} \sum_{i=1}^{M} x_{ui} (\Theta^*_{ui} - \bar{p}_i^t) \right\} \leq \sum_{u=1}^{N} \sum_{i=1}^{M} x_{ui}^t (\Theta^t_{ui} - \bar{p}_i^t)
    \end{equation*}
    where $\mathcal{X}_u^t = \{ x \in \mathds{R}^{M}_{+} : \mathds{1}^\textrm{T}x \leq d_u^t\}$.
\label{lemma_ofu_instability_bound}
\end{lemma}
\begin{proof}
    From the definition of $\vect{\Theta}_t$ and $\vect{X}_t$, we have
    \begin{align*}
         \sum_{u=1}^{N} \sum_{i=1}^{M} x_{ui}^t \Theta^t_{ui} &= \langle \vect{X}_t, \vect{\Theta}_t \rangle \\
        &= \max_{(\vect{X}, \vect{\Theta}) \in \mathcal{X}_t \times \mathcal{Q}_t} \langle \vect{X}, \vect{\Theta} \rangle \\
        &= \max_{\vect{\Theta} \in \mathcal{Q}_t}  \max_{\vect{X} \in \mathcal{X}_t} \langle \vect{X}, \vect{\Theta} \rangle \\
        &= \begin{aligned}
        \max_{\vect{\Theta} \in \mathcal{Q}_t} \max_{\vect{X} \in \mathds{R}^{N \times M}_{+}} & \min_{\vect{\lambda} \in \mathds{R}^{M}_{+}} \langle \vect{X}, \vect{\Theta} \rangle + \vect{\lambda}^\textrm{T} (\vect{c}_t - \vect{X}^\textrm{T} \mathds{1}) \\
        \text{s.t.} &\; \vect{X} \mathds{1} \leq \vect{d}^t\\
        \end{aligned}\\
        &= \begin{aligned}
        \max_{\vect{\Theta} \in \mathcal{Q}_t} \max_{\vect{X} \in \mathds{R}^{N \times M}_{+}} & \min_{\vect{\lambda} \in \mathds{R}^{M}_{+}} \langle \vect{X}, \vect{\Theta} -  \mathds{1} \vect{\lambda}^\textrm{T} \rangle + \vect{\lambda}^\textrm{T} \vect{c}_t \\
        \text{s.t.} &\; \vect{X} \mathds{1} \leq \vect{d}^t\\
        \end{aligned}\\
        &= \begin{aligned}
        \max_{\vect{\Theta} \in \mathcal{Q}_t} \max_{\vect{X} \in \mathds{R}^{N \times M}_{+}} & \min_{\vect{\lambda} \in \mathds{R}^{M}_{+}} \sum_{u=1}^{N}  \sum_{i=1}^{M} x_{ui} (\Theta_{ui} - \lambda_i) + \sum_{i=1}^{M} \lambda_i c^t_i \\
        \text{s.t.} &\; \vect{X} \mathds{1} \leq \vect{d}^t\\
        \end{aligned}\\
    \end{align*}

    Note that the optimum value of this optimization problem is achieved when $\vect{\lambda} = \bar{\vect{p}}_t$, $\vect{\Theta} = \vect{\Theta}_t$ and $\vect{X} = \vect{X}_t$. Therefore, if we fix $\vect{\lambda} = \bar{\vect{p}}_t$ and optimize over $\vect{\Theta}$ and $\vect{X}$, the optimum value of the objective does not change. Consequently, we can write
    
    \begin{align*}
        \sum_{u=1}^{N} \sum_{i=1}^{M} x_{ui}^t \Theta^t_{ui} &= \begin{aligned}
        \max_{\vect{\Theta} \in \mathcal{Q}_t} \max_{\vect{X} \in \mathds{R}^{N \times M}_{+}} & \sum_{u=1}^{N}  \sum_{i=1}^{M} x_{ui} (\Theta_{ui} - \bar{p}_i^t) + \sum_{i=1}^{M} \bar{p}_i^t c^t_i \\
        \text{s.t.} &\; \vect{X} \mathds{1} \leq \vect{d}^t\\
        \end{aligned}\\
        &= \max_{\vect{\Theta} \in \mathcal{Q}_t} \sum_{u=1}^{N} \left\{ \max_{\vect{x}_u \in \mathcal{X}_u^t} \sum_{i=1}^{M} x_{ui} (\Theta_{ui} - \bar{p}_i^t) \right \} + \sum_{i=1}^{M} \bar{p}_i^t c^t_i\\
        &\geq \sum_{u=1}^{N} \left\{ \max_{\vect{x}_u \in \mathcal{X}_u^t} \sum_{i=1}^{M} x_{ui} (\Theta^*_{ui} - \bar{p}_i^t) \right \} + \sum_{i=1}^{M} \bar{p}_i^t c^t_i
    \end{align*}
    where the last step uses the condition that $\vect{\Theta}^* \in \mathcal{Q}_t$. Noting that $\left (c^t_i - \sum_{u=1}^{N} x_{ui}^t \right) \bar{p}_i^t = 0$ by complementary slackness, we can write
    \begin{align*}
        \sum_{u=1}^{N} \left\{ \max_{\vect{x}_u \in \mathcal{X}_u^t} \sum_{i=1}^{M} x_{ui} (\Theta^*_{ui} - \bar{p}_i^t) \right \} &\leq \sum_{u=1}^{N} \sum_{i=1}^{M} x_{ui}^t \Theta^t_{ui} - \sum_{i=1}^{M} \bar{p}_i^t c^t_i\\
        &= \sum_{u=1}^{N} \sum_{i=1}^{M} x_{ui}^t (\Theta^t_{ui} - \bar{p}_i^t)
    \end{align*}

\end{proof}

\begin{lemma}
In the model without acceptance/rejection options, if the offered prices are $\vect{p}_t = \bar{\vect{p}}_t$, and $\vect{\Theta}^* \in \mathcal{Q}_t$ for all $t \leq T$, then
\begin{equation*}
    \sum_{u = 1}^{N} \mathcal{S}^{t}_{u} \leq \langle \vect{X}_t, \vect{\Theta}_{t} - \vect{\Theta}^{*} \rangle
\end{equation*}

\label{instability_to_r}
\end{lemma}

\begin{proof}
Using the definition of the instability and Lemma \ref{lemma_ofu_instability_bound}, we have
\begin{align*}
    \sum_{u = 1}^{N} \mathcal{S}^{t}_{u} &= \sum_{u = 1}^{N} \left \{ \max_{\vect{x} \in \mathcal{X}_u^t} \mathds{E} \left[ \sum_{i=1}^{M} x_{ui} (R^t_{ui} - p_i^t) \right] - \mathds{E} \left[\sum_{i=1}^{M} x_{ui}^t (R^t_{ui} - p_i^t) \right ] \right\}\\
    &\leq \sum_{u = 1}^{N} \left \{ \max_{\vect{x} \in \mathcal{X}_u^t} \sum_{i=1}^{M} x_{ui} (\Theta^*_{ui} - p_i^t) - \sum_{i=1}^{M} x_{ui}^t (\Theta^*_{ui} - p_i^t) \right \} \\
    &\leq \sum_{u = 1}^{N} \left \{ \sum_{i=1}^{M} x_{ui}^t (\Theta^t_{ui} - p_i^t) - \sum_{i=1}^{M} x_{ui}^t (\Theta^*_{ui} - p_i^t) \right \} \\
    &\leq \sum_{u = 1}^{N} \sum_{i=1}^{M} x_{ui}^t (\Theta^t_{ui} - \Theta^*_{ui}) \\
    &= \langle \vect{X}_t, \vect{\Theta}_{t} - \vect{\Theta}^{*} \rangle \\
\end{align*} 
\end{proof}

\begin{theorem}
In the model without acceptance/rejection options, if the offered prices are $\vect{p}_t = \bar{\vect{p}}_t$, and $\vect{\Theta}^* \in \mathcal{Q}_t$ for all $t \leq T$, then the T period instability is bounded as
    \begin{equation}
        \mathcal{R}^{I}(T, \pi)  \leq \sqrt{ 8 NM \beta_T^*(\delta, \alpha, \gamma) T \log \left(1 + \frac{T}{\gamma} \right) }
    \end{equation}
\label{thm_instability}
\end{theorem}

\begin{proof}
Using Lemma \ref{instability_to_r} and \ref{lemma_regret_ub},
\begin{align*}
     \mathcal{R}^{I}(T, \pi) &\leq \sum_{t=1}^{T} \langle \vect{X}_t, \vect{\Theta}_{t} - \vect{\Theta}^{*} \rangle \\
    &\leq 2 \sqrt{\beta_T} \sum_{t=1}^{T} w_t \\
    &\leq 2 \sqrt{\beta_T T \sum_{t=1}^{T} w_t^2 } \\
    &\leq \sqrt{ 8 \beta_T NM T \log \left(1 + \frac{T}{\gamma} \right) }
\end{align*}

\end{proof}

\begin{corollary}
In the model without acceptance/rejection options, letting $\delta = \mathcal{O} ((NMT)^{-1})$, $\alpha = \mathcal{O} ((NMT)^{-1})$ and $\gamma = 1$, Algorithm \ref{alg_low} achieves regret that satisfy
\begin{align}
    \mathcal{R}^{SW}(T, \pi) &= \widetilde{\mathcal{O}} \left( \sqrt{N M (N+M) R T} \right) \\
    \mathcal{R}^{I}(T, \pi) &= \widetilde{\mathcal{O}} \left( \sqrt{N M (N+M) R T} \right)
\end{align}
\label{corr_regret_order}
\end{corollary}

\begin{proof}
By Lemma \ref{lemma_confidence}, $\vect{\Theta}^* \in \mathcal{Q}_t$ for all $t \leq T$ with probability at least $1 - 2 \delta$. Therefore, the bounds of Theorem \ref{thm_regret} (or Theorem \ref{thm_instability} for instability) hold true with probability at least $1 - 2 \delta$. Then, both for social welfare regret and instability regret,
\begin{equation*}
    \mathcal{R}(T, \pi) \leq (1 - \delta) \sqrt{ 8 d \beta_T^* (\delta, \alpha, \gamma) T \log \left(1 + \frac{T}{\gamma} \right) } + 2 \delta BdT
\end{equation*}

Letting $\delta = \mathcal{O} (T^{-1})$, $\alpha = \mathcal{O} = (T^{-1})$ and $\gamma = 1$,
\begin{equation*}
    \mathcal{R}(T, \pi) = \widetilde{\mathcal{O}} \left( \sqrt{NM \beta_T^* (T, T^{-1}, 1) T} \right)
\end{equation*}

Then, we note that $\beta_T^* (T, T^{-1}, 1) = \widetilde{\mathcal{O}} \left(\eta^2 \log \left(\mathcal{N}(\mathcal{F}, T^{-1}, \| \cdot \|_{2})\right) \right)$. Using Lemma \ref{lemma_covering} to upper bound the covering number, we complete the proof.

\end{proof}

\subsection{Extension to the Model with Acceptance/Rejection Options}

\begin{lemma}
    For any $\nu \geq 0$, if $\vect{\Theta}^* \in \mathcal{Q}_t$, then
    \begin{equation*}
        \sum_{(u,i) \in \mathcal{A}_t} \mathbb{1} \left \{ \Delta^t_{ui} \geq \nu \right \} \leq \frac{2 w_t \sqrt{\beta_t}}{\nu}
    \end{equation*}
\end{lemma}

\begin{proof}
     Using Lemma \ref{lemma_regret_ub},
    \begin{align*}
        \sum_{(u,i) \in \mathcal{A}_t} \mathbb{1} \left \{ \Delta^t_{ui} \geq \nu \right \}
        &\leq \frac{1}{\nu} \sum_{(u,i) \in \mathcal{A}_t} \left| \Delta^t_{ui} \right| \\
        &\leq \frac{2 w_t \sqrt{\beta_t}}{\nu}
    \end{align*}
\end{proof}

\begin{theorem}    
In the model with acceptance/rejection options, if the offered prices $\vect{p}_t$ satisfy $\vect{p}_t = \bar{\vect{p}}_t - \nu_t$ where $\nu_t = \nu \sqrt{w_t}$ for some $\nu \geq 0$, and $\vect{\Theta}^* \in \mathcal{Q}_t$ for all $t \leq T$, then the $T$ period social welfare regret is bounded as
\begin{equation*}
    \mathcal{R}^{SW}(T, \pi) \leq \left( 8 NM \beta_T T \log \left(1 + \frac{T}{\gamma} \right) \right)^{1/2} + \frac{2}{\nu} \left( 2 NM T^3 \beta_T^2 \log \left(1 + \frac{T}{\gamma} \right) \right)^{1/4}
\end{equation*}
\label{thm_regret_ar}
\end{theorem}

\begin{proof}

In the model with acceptance/rejection options, let the instantaneous regret at time $t$ be defined as 
\begin{equation*}
    \widetilde{r}_t := \sum_{(u,i) \in \mathcal{A}^*_t} \Theta^*_{ui} - \sum_{(u,i) \in \mathcal{A}_t} \Theta^*_{ui} \mathds{1}\{ \Theta^*_{ui} \geq p_i^t \}
\end{equation*}
Then, recalling the definition of $r_t$, we can write
\begin{align*}
    \widetilde{r}_t &= \sum_{(u,i) \in \mathcal{A}^*_t} \Theta^*_{ui} - \sum_{(u,i) \in \mathcal{A}_t} \Theta^*_{ui} + \sum_{(u,i) \in \mathcal{A}_t} \Theta^*_{ui} \mathds{1}\{ \Theta^*_{ui} < p_i^t \} \\
    &\leq \langle \vect{X}_t^*, \vect{\Theta}^* \rangle - \langle \vect{X}_t, \vect{\Theta}^* \rangle + \sum_{(u,i) \in \mathcal{A}_t} \mathds{1}\{ \Theta^*_{ui} < p_i^t \} \\
    &\leq r_t + \sum_{(u,i) \in \mathcal{A}_t} \mathds{1}\{ \Theta^{t}_{ui} - \Theta^{*}_{ui} > \Theta^{t}_{ui} - p_i^t \} \\
    &\leq r_t + \sum_{(u,i) \in \mathcal{A}_t} \mathds{1}\{ \Delta^t_{ui} > \Theta^{t}_{ui} - \bar{p}_i^t + \nu_t \} \\
    &\leq r_t + \sum_{(u,i) \in \mathcal{A}_t} \mathds{1}\{ \Delta^t_{ui} > \nu_t \}
\end{align*}
where the last step uses the fact that $\Theta^{t}_{ui} \geq \bar{p}_i^t$ for $(u,i) \in \mathcal{A}_t$. Then, we have 
\begin{align*}
    \mathcal{R}^{SW}(T, \pi) &= \sum_{t = 1}^{T} \widetilde{r}_t \\
    &\leq \sum_{t = 1}^{T} r_t + \sum_{t = 1}^{T} \sum_{(u,i) \in \mathcal{A}_t} \mathds{1}\{ \Delta^t_{ui} > \nu_t \}\\
    &\leq \left( T \sum_{t = 1}^{T} r_t^2 \right)^{1/2} + 2 \sqrt{\beta_T} \sum_{t = 1}^{T} \frac{w_t}{\nu_t}\\
    &= \left( T \sum_{t = 1}^{T} r_t^2 \right)^{1/2} + \frac{2}{\nu} \sqrt{\beta_T} \sum_{t = 1}^{T} \sqrt{w_t}\\
    &\leq \left( T \sum_{t = 1}^{T} r_t^2 \right)^{1/2} + \frac{2}{\nu} \sqrt{\beta_T} \left( T^3 \sum_{t = 1}^{T} w_t^2 \right)^{1/4}\\
    &\leq \left( 8 NM \beta_T T \log \left(1 + \frac{T}{\gamma} \right) \right)^{1/2} + \frac{2}{\nu} \left( 2 NM T^3 \beta_T^2 \log \left(1 + \frac{T}{\gamma} \right) \right)^{1/4}
\end{align*}
\end{proof}

\begin{lemma}
In the model with acceptance/rejection options, if the offered prices $\vect{p}_t$ satisfy $\vect{p}_t = \bar{\vect{p}}_t - \nu_t$ where $\nu_t = \nu \sqrt{w_t}$ for some $\nu \geq 0$, and $\vect{\Theta}^* \in \mathcal{Q}_t$, then
\begin{equation*}
    \sum_{u = 1}^{N} \mathcal{S}^{t}_{u} \leq \langle \vect{X}_t, \vect{\Theta}_{t} - \vect{\Theta}^{*} \rangle + \nu_t n M
\end{equation*}
\label{instability_to_r_accept_reject}
\end{lemma}

\begin{proof}
Using the definition of the instability and Lemma \ref{lemma_ofu_instability_bound}, we have
\begin{align*}
    \sum_{u = 1}^{N} \mathcal{S}^{t}_{u} &=  \sum_{u = 1}^{N} \left \{ \max_{\vect{x} \in \mathcal{X}_u^t} \sum_{i=1}^{M} x_{ui} (\Theta^*_{ui} - p_i^t) \mathds{1}\{ \Theta^*_{ui} \geq p_i^t \} - \sum_{i=1}^{M} x_{ui}^t (\Theta^*_{ui} - p_i^t) \mathds{1}\{ \Theta^*_{ui} \geq p_i^t \} \right \} \\
    &\leq \sum_{u = 1}^{N} \left \{ \max_{\vect{x} \in \mathcal{X}_u^t} \sum_{i=1}^{M} x_{ui} (\Theta^*_{ui} - \bar{p}_i^t + \nu_t) \mathds{1}\{ \Theta^*_{ui} \geq \bar{p}_i^t -\nu_t \} - \sum_{i=1}^{M} x_{ui}^t (\Theta^*_{ui} - p_i^t) \mathds{1}\{ \Theta^*_{ui} \geq p_i^t \} \right \} \\
    &\leq \sum_{u = 1}^{N} \left \{ \max_{\vect{x} \in \mathcal{X}_u^t} \sum_{i=1}^{M} x_{ui} (\Theta^*_{ui} - \bar{p}_i^t ) \mathds{1}\{ \Theta^*_{ui} \geq \bar{p}_i^t -\nu_t \} +  \nu_t \max_{\vect{x} \in \mathcal{X}_u^t} \sum_{i=1}^{M} x_{ui} - \sum_{i=1}^{M} x_{ui}^t (\Theta^*_{ui} - p_i^t) \mathds{1}\{ \Theta^*_{ui} \geq p_i^t \} \right \} \\
    &\leq \sum_{u = 1}^{N} \left \{ \max_{\vect{x} \in \mathcal{X}_u^t} \sum_{i=1}^{M} x_{ui} (\Theta^*_{ui} - \bar{p}_i^t ) \mathds{1}\{ \Theta^*_{ui} \geq \bar{p}_i^t \} +  \nu_t d_u^t - \sum_{i=1}^{M} x_{ui}^t (\Theta^*_{ui} - p_i^t) \mathds{1}\{ \Theta^*_{ui} \geq p_i^t \} \right \} \\
    &\leq \sum_{u = 1}^{N} \left \{ \max_{\vect{x} \in \mathcal{X}_u^t} \sum_{i=1}^{M} x_{ui} (\Theta^*_{ui} - \bar{p}_i^t ) +  \nu_t d_u^t - \sum_{i=1}^{M} x_{ui}^t (\Theta^*_{ui} - p_i^t) \mathds{1}\{ \Theta^*_{ui} \geq p_i^t \} \right \} \\
    &\leq \sum_{u = 1}^{N} \left \{ \sum_{i=1}^{M} x_{ui}^t (\Theta^t_{ui} - \bar{p}_i^t) +  \nu_t d_u^t - \sum_{i=1}^{M} x_{ui}^t (\Theta^*_{ui} - p_i^t) \right \} \\
    &\leq \sum_{u = 1}^{N} \left \{ \sum_{i=1}^{M} x_{ui}^t (\Theta^t_{ui} - \Theta^*_{ui}) +  \nu_t d_u^t \right \} \\
    &\leq \sum_{u = 1}^{N} \sum_{i=1}^{M} x_{ui}^t \Delta^t_{ui} + \nu_t  \sum_{u = 1}^{N} d_u^t \\
    &\leq \langle \vect{X}_t, \vect{\Theta}_{t} - \vect{\Theta}^{*} \rangle + \nu_t n M
\end{align*}
    
\end{proof}

\begin{theorem}    
In the model with acceptance/rejection options, if the offered prices $\vect{p}_t$ satisfy $\vect{p}_t = \bar{\vect{p}}_t - \nu_t$ where $\nu_t = \nu \sqrt{w_t}$ for some $\nu \geq 0$, and $\vect{\Theta}^* \in \mathcal{Q}_t$ for all $t \leq T$, then the $T$ period instability is bounded as
\begin{equation*}
    \mathcal{R}^{I}(T, \pi) \leq \left(8 \beta_T NMT \log \left(1 + \frac{T}{\gamma} \right) \right)^{1/2} + \nu n M \left( 2 NM T^3 \log \left(1 + \frac{T}{\gamma} \right) \right)^{1/4}
\end{equation*}
\label{thm_instability_ar}
\end{theorem}

\begin{proof}
Using Lemma \ref{instability_to_r_accept_reject} and \ref{lemma_regret_ub},
\begin{align*}
     \mathcal{R}^{I}(T, \pi) &\leq \sum_{t=1}^{T} \left( \langle \vect{X}_t, \vect{\Theta}_{t} - \vect{\Theta}^{*} \rangle + \nu_t n M \right) \\
    &\leq 2 \sqrt{\beta_T} \sum_{t=1}^{T} w_t + n M \sum_{t=1}^{T} \nu_t\\
    &\leq \sqrt{ 8 \beta_T NMT \log \left(1 + \frac{T}{\gamma} \right) } + \nu n M \sum_{t=1}^{T} \sqrt{w_t}\\
    &\leq \sqrt{ 8 \beta_T NMT \log \left(1 + \frac{T}{\gamma} \right) } + \nu n M \sum_{t=1}^{T} \sqrt{w_t}\\
    &\leq \left(8 \beta_T NMT \log \left(1 + \frac{T}{\gamma} \right) \right)^{1/2} + \nu n M \left( 2 NM T^3 \log \left(1 + \frac{T}{\gamma} \right) \right)^{1/4}
\end{align*}

\end{proof}

\begin{theorem}  
In the model with acceptance/rejection options, if the offered prices $\vect{p}_t$ satisfy $\vect{p}_t = \bar{\vect{p}}_t - \nu_t$ where $\nu_t = \nu \sqrt{w_t}$ for $\nu = \left( \frac{4 \beta_T}{N^2 M^2} \right)^{1/4}$, then, with probability $1 - 2 \delta$, the algorithm results in regrets that satisfy
\begin{align*}
    \mathcal{R}^{SW}(T, \pi) &\leq \sqrt{\lambda_T NM T} + (\lambda_T)^{\frac{1}{4}} (NMT)^{\frac{3}{4}} \\
    \mathcal{R}^{I}(T, \pi) &\leq \sqrt{\lambda_T NM T} + (\lambda_T)^{\frac{1}{4}} (NMT)^{\frac{3}{4}}
\end{align*}
where $\lambda_T = 8 \beta_T \log (1 + T/\gamma)$.
\label{thm_regret_and_instability_ar}
\end{theorem}

\begin{proof}
By Theorem \ref{thm_regret_ar},
\begin{align*}
    \mathcal{R}^{SW}(T, \pi) &\leq \left( 8 NM \beta_T T \log \left(1 + \frac{T}{\gamma} \right) \right)^{1/2} + \frac{2}{\nu} \left( 2 NM T^3 \beta_T^2 \log \left(1 + \frac{T}{\gamma} \right) \right)^{1/4} \\
    &=  (\lambda_T NM T)^{\frac{1}{2}} + (\lambda_T)^{\frac{1}{4}} (NMT)^{\frac{3}{4}} \\
\end{align*}

Similarly, we have
\begin{align*}
    \mathcal{R}^{I}(T, \pi) &\leq  \left(8 \beta_T NMT \log \left(1 + \frac{T}{\gamma} \right) \right)^{1/2} + \nu N M \left( 2 NM T^3 \log \left(1 + \frac{T}{\gamma} \right) \right)^{1/4} \\
    &= (\lambda_T NM T)^{\frac{1}{2}} + (\lambda_T)^{\frac{1}{4}} (NMT)^{\frac{3}{4}}
\end{align*}

\end{proof}

\begin{theorem}    
In the model with acceptance/rejection options, Letting $\delta = \mathcal{O} ((NMT)^{-1})$, $\alpha = \mathcal{O} ((NMT)^{-1})$, $\gamma = 1$ and $\nu = \mathcal{O} \left(\left( \frac{\eta^2 (N+M) R}{N^2 M^2}\right)^{1/4}\right)$ results in regret bounds that satisfy
\begin{align*}
    \mathcal{R}^{SW}(T, \pi) = \widetilde{\mathcal{O}} \left( \left( \eta^2 (N+M) R \right)^{\frac{1}{4}} (NMT)^{\frac{3}{4}}  \right) \\
    \mathcal{R}^{I}(T, \pi) = \widetilde{\mathcal{O}} \left( \left( \eta^2 (N+M) R \right)^{\frac{1}{4}} (NMT)^{\frac{3}{4}}  \right)
\end{align*}
\end{theorem}

\begin{proof}
By Theorem \ref{thm_regret_and_instability_ar}, with probability $1$,
\begin{equation*}
    \mathcal{R}^{SW}(T, \pi) \leq (1 - 2 \delta) \left((\lambda_T NM T)^{\frac{1}{2}} + (\lambda_T)^{\frac{1}{4}} (NMT)^{\frac{3}{4}} \right)  + 2 \delta NMT
\end{equation*}

Similarly, with probability $1$, we have
\begin{equation*}
    \mathcal{R}^{I}(T, \pi) \leq (1 - 2 \delta)  \left ( (\lambda_T NM T)^{\frac{1}{2}} + (\lambda_T)^{\frac{1}{4}} (NMT)^{\frac{3}{4}} \right) + 2 \delta NMT
\end{equation*}

Letting $\delta = \mathcal{O} ((NMT)^{-1})$, $\alpha = \mathcal{O} ((NMT)^{-1})$, $\gamma = 1$, and noting that $\beta_T^* (T, T^{-1}, 1) = \widetilde{\mathcal{O}} \left(\eta^2 \log \left(\mathcal{N}(\mathcal{L}, T^{-1}, \| \cdot \|_{2})\right) \right) = \widetilde{\mathcal{O}} \left( \eta^2 (N+M) R \right)$ we conclude the proof.

\end{proof}

\subsection{Covering Number for Low-rank Matrices}

\begin{lemma}
The covering number of $\mathcal{L}$ given in \eqref{low_l} obeys
\begin{equation}
    \log \mathcal{N}(\mathcal{L}, \alpha, \| \cdot \|_{\mathrm{F}}) \leq (N + M + 1) R \log \left( \frac{9 \sqrt{NM}}{\alpha} \right)
\end{equation}
\label{lemma_covering}
\end{lemma}

\begin{proof}
This proof is modified from \citep{candes_2011}. Let $\mathcal{S} = \{ \vect{\Theta} \in \mathds{R}^{N \times M} : \text{rank}(\vect{\Theta}) \leq R, \|\vect{\Theta}\|_\mathrm{F} \leq 1\}$. We will first show that there exists an $\epsilon$-net $\mathcal{S}^\epsilon$ for the Frobenious norm obeying 
\begin{equation*}
    |\mathcal{S}^\epsilon| \leq \left(9 / \epsilon \right)^{(N+M+1)R}
\end{equation*} 

For any $\vect{\Theta} \in \mathcal{S}$, singular value decomposition gives $\vect{\Theta} = \vect{U} \vect{\Sigma} \vect{V}^\textrm{T}$, where $\|\vect{\Sigma}\|_{\mathrm{F}} \leq 1$. We will construct an $\epsilon$-net for $\mathcal{S}$ by covering the set of permisible $\vect{U}$, $\vect{\Sigma}$ and $\vect{V}$. Let $\mathcal{D}$ be the set of diagonal matrices with nonnegative diagonal entries and Frobenious norm less than or equal to one. We take $\mathcal{D}^{\epsilon/3}$ be an $\epsilon$/3-net for $\mathcal{D}$ with $|\mathcal{D}^{\epsilon/3}| \leq (9/\epsilon)^R$. Next, let $\mathcal{O}_{N,R} = \{\vect{U} \in \mathbb{R}^{N \times R} : \vect{U}^\textrm{T} \vect{U} = \vect{I} \}$. To cover $\mathcal{O}_{N,R}$, we use the $\|\cdot\|_{1,2}$ norm defined as
\begin{equation*}
    \|\vect{U}\|_{1,2} = \max_{i} \|\vect{u_i}\|_{\ell_2}
\end{equation*}
where $\vect{u_i}$ denotes the $i$th column of $\vect{U}$. Let $\mathcal{Q}_{N, R} = \{\vect{U} \in \mathbb{R}^{N, R} : \|\vect{U}\|_{1,2} \leq 1\}$. It is easy to see that $\mathcal{O}_{{N, R}} \subset \mathcal{Q}_{{N, R}}$ since the columns of an orthogonal matrix are unit normed. We see that there is an $\epsilon/3$-net $\mathcal{O}_{{N, R}}^{\epsilon/3}$ for $\mathcal{O}_{{N, R}}$ obeying $|\mathcal{O}_{{N, R}}^{\epsilon/3}| \leq (9/\epsilon)^{NR}$. Similarly, let $\mathcal{P}_{M,R} = \{\vect{V} \in \mathbb{R}^{M \times R} : \vect{V}^\textrm{T} \vect{V} = \vect{I} \}$.  By the same argument, there is an $\epsilon/3$-net $\mathcal{P}_{M, R}^{\epsilon/3}$ for $\mathcal{P}_{M, R}$ obeying $|\mathcal{P}_{{M, R}}^{\epsilon/3}| \leq (9/\epsilon)^{MR}$. We now let $\mathcal{S}^\epsilon = \{ \bar{\vect{U}} \bar{\vect{\Sigma}} \bar{\vect{V}}^\textrm{T} : \bar{\vect{U}} \in \mathcal{O}_{{N, R}}^{\epsilon/3}, \bar{\vect{V}} \in \mathcal{P}_{{M, R}}^{\epsilon/3}, \bar{\vect{\Sigma}} \in \mathcal{D}^{\epsilon/3}\}$, and remark $|\mathcal{S}^\epsilon| \leq |\mathcal{O}_{{N, R}}^{\epsilon/3}| |\mathcal{P}_{{M, R}}^{\epsilon/3}| |\mathcal{D}^{\epsilon/3}| \leq (9/\epsilon)^{(N+M+1)R}$. It remains to show that for all $\vect{\Theta} \in \mathcal{S}$, there exists $\bar{\vect{\Theta}} \in \mathcal{S}^\epsilon$ with $\|\vect{\Theta} - \bar{\vect{\Theta}}\|_\mathrm{F} \leq \epsilon$.

Fix $\vect{\Theta} \in \mathcal{S}$ and decompose it as $\vect{\Theta} = \vect{U} \vect{\Sigma} \vect{V}^\textrm{T}$. Then, there exists $\bar{\vect{\Theta}} = \bar{\vect{U}} \bar{\vect{\Sigma}} \bar{\vect{V}}^\textrm{T} \in \mathcal{S}^\epsilon$ with $\bar{\vect{U}} \in \mathcal{O}_{{N, R}}^{\epsilon/3}, \bar{\vect{V}} \in \mathcal{P}_{{M, R}}^{\epsilon/3}, \bar{\vect{\Sigma}} \in \mathcal{D}^{\epsilon/3}$ satisfying $\|\vect{U} - \bar{\vect{U}}\|_{1, 2} \leq \epsilon/3$, $\|\vect{V} - \bar{\vect{V}}\|_{1, 2} \leq \epsilon/3$ and $\|\vect{\Sigma} - \bar{\vect{\Sigma}}\|_{\mathrm{F}} \leq \epsilon/3$. This gives 
\begin{align*}
    \|\vect{\Theta} - \bar{\vect{\Theta}}\|_{\mathrm{F}} &= \|\vect{U} \vect{\Sigma} \vect{V}^\textrm{T} - \bar{\vect{U}} \bar{\vect{\Sigma}} \bar{\vect{V}}^\textrm{T}\|_{\mathrm{F}}\\
    &= \|\vect{U} \vect{\Sigma} \vect{V}^\textrm{T} - \bar{\vect{U}} \vect{\Sigma} \vect{V}^\textrm{T} + \bar{\vect{U}} \vect{\Sigma} \vect{V}^\textrm{T} - \bar{\vect{U}} \bar{\vect{\Sigma}} \vect{V}^\textrm{T} + \bar{\vect{U}} \bar{\vect{\Sigma}} \vect{V}^\textrm{T} - \bar{\vect{U}} \bar{\vect{\Sigma}} \bar{\vect{V}}^\textrm{T}\|_{\mathrm{F}} \\
    &\leq \|(\vect{U} - \bar{\vect{U}}) \vect{\Sigma} \vect{V}^\textrm{T}\|_{\mathrm{F}} + \|\bar{\vect{U}} (\vect{\Sigma} - \bar{\vect{\Sigma}}) \vect{V}^\textrm{T}\|_{\mathrm{F}} + \|\bar{\vect{U}} \bar{\vect{\Sigma}} (\vect{V} - \bar{\vect{V}})^\textrm{T} \|_{\mathrm{F}}
\end{align*}

For the first term, since $\vect{V}$ is an orthogonal matrix,
\begin{align*}
    \|(\vect{U} - \bar{\vect{U}}) \vect{\Sigma} \vect{V}^\textrm{T}\|^2_{\mathrm{F}} &= \|(\vect{U} - \bar{\vect{U}}) \vect{\Sigma} \|^2_{\mathrm{F}} \\
    &\leq \|\Sigma\|^2_{\mathrm{F}} \|\vect{U} - \bar{\vect{U}} \|^2_{1, 2} \leq (\epsilon/3)^2
\end{align*}
By the same argument, $\|\bar{\vect{U}} \bar{\vect{\Sigma}} (\vect{V} - \bar{\vect{V}})^\textrm{T} \|_{\mathrm{F}} \leq \epsilon/3$ as well. Lastly, $\|\bar{\vect{U}} (\vect{\Sigma} - \bar{\vect{\Sigma}}) \vect{V}^\textrm{T}\|_{\mathrm{F}} = \|\vect{\Sigma} - \bar{\vect{\Sigma}}\|_{\mathrm{F}} \leq \epsilon / 3$. Therefore, $\|\vect{\Theta} - \bar{\vect{\Theta}}\|_{\mathrm{F}} \leq \epsilon$, showing that $\mathcal{S}^\epsilon$ is an $\epsilon$-net for $\mathcal{S}$ with respect to the Frobenious norm.

Next, we will construct an $\alpha$-net for $\mathcal{L}$ given in equation \ref{low_l}. Let $\kappa = \sqrt{NM}$. We start by noting that for all $\vect{\Theta} \in \mathcal{L}$, the Frobenious norm obeys $\|\vect{\Theta}\|_\mathrm{F} \leq \kappa$. Then, define $\vect{X} = \frac{1}{\kappa} \vect{\Theta} \in \mathcal{S}$ and $\mathcal{L}^\alpha := \left\{ \kappa \bar{\vect{X}} : \bar{\vect{X}} \in \mathcal{S}^\epsilon \right\}$. We previously showed that for any $\vect{X} \in \mathcal{S}$,  there exists $\bar{\vect{X}} \in \mathcal{S}^\epsilon$ such that $\|\vect{X} - \bar{\vect{X}}\|_\mathrm{F} \leq \epsilon$. Therefore, for any $\vect{\Theta} \in \mathcal{L}$,  there exists $\bar{\vect{\Theta}} = \kappa \bar{\vect{X}} \in \mathcal{L}^\alpha$ such that $\|\vect{\Theta} - \bar{\vect{\Theta}}\|_\mathrm{F} \leq \kappa \epsilon$. Setting $\epsilon = \alpha / \kappa$, we obtain that $\mathcal{L}^\alpha$ is an $\alpha$-net for $\mathcal{L}$ with respect to the Frobenious norm. Finally, the size of $\mathcal{L}^\alpha$ obeys
\begin{equation*}
    |\mathcal{L}^\alpha| = |\mathcal{S}^{\alpha / \kappa}| \leq \left(9 \kappa / \alpha \right)^{(N+M+1)R}
\end{equation*} 
This completes the proof.
\end{proof}

\section{PROOFS FOR CONTEXTUAL ALLOCATION AND PRICING}

Let $\vect{\Theta}^* \in \mathcal{F} := \{ \vect{\Theta} \in \mathbb{R}^{M \times R} : \exists \vect{F} \in \mathbb{R}^{N \times R}, \vect{\Theta} = \vect{F} \vect{\Phi}^\mathrm{T}, \|\vect{F}\|_{2, \infty} \leq 1\}$ for some known item feature matrix $\vect{\Phi} \in \mathbb{R}^{M \times R}$ that satisfy $\|\vect{\Phi}\|_{2, \infty} \leq 1$. This condition is the same as having each row $[\vect{\Theta}^*]_{u,:}$ belonging to the set $ \widetilde{\mathcal{F}} := \{ \vect{\theta} : \exists \vect{f} \in \mathds{R}^{R}, \vect{\theta} = \vect{\Phi} \vect{f}, \|\vect{f}\|_2 \leq 1\}$. In other words,
\begin{equation*}
    \vect{\Theta}^* \in \mathcal{F} \iff  [\vect{\Theta}^*]_{u,:} \in \widetilde{\mathcal{F}} , \forall u \in \mathcal{N}
\end{equation*}

Hence, we can construct separate confidence sets for each row of the matrix $\vect{\Theta}^*$. The confidence sets that we construct are centered around the regularized least square estimates. For each user $u$, we let the cumulative squared prediciton error at time $t$ be
\begin{equation}
    L_{2,u}^{t}(\vect{\theta}) = \sum_{\tau=1}^{t-1} \sum_{i : (u,i) \in \mathcal{A}_\tau} (\theta_i  - R^{\tau}_{ui})^2
\end{equation}
and define the regularized least squares estimate at time $t$ as
\begin{equation}
    \widehat{\vect{\theta}}^t_u = \argmin_{\vect{\theta} \in \widetilde{\mathcal{F}}} \left \{ L^{t}_{2,u}(\vect{\theta}) + \gamma \|\vect{\theta} - \vect{\theta}^{\circ}_{u} \|_2^2 \right \}
\end{equation}

where $\vect{\theta}^{\circ}_{u} = [\vect{\Theta}_{\circ}]_{u, :}$ is the $u$-th row of $\vect{\Theta}_{\circ}$. Then, the confidence sets take the form $\mathcal{C}^t_u := \{ \vect{\theta} \in \widetilde{\mathcal{F}} : \|\vect{\theta} - \widehat{\vect{\theta}}^t_u \|_{E^{t}_{2,u}} \leq \sqrt{\rho_t}\}$ where $\rho_t$ is an appropriately chosen confidence parameter, and the regularized empirical $L_2$-norm $\| \cdot \|_{E^{t}_{2, u}}$ of user $u$ is defined by 
\begin{equation*}
    \| \vect{\delta} \|_{E^{t}_{2, u}}^2 := \sum_{\tau=1}^{t-1} \sum_{i : (u, i) \in \mathcal{A}_\tau} \langle \vect{\delta}, \vect{e}_{i} \rangle^2 + \gamma \| \vect{\delta} \|_{2}^2 =  \sum_{i=1}^{d} (n^{t}_{ui} + \gamma) (\delta_{i})^2
\end{equation*}

Similar to our analysis in the low-rank setting, we start with providing the following guarantee for the confidence regions $\mathcal{C}^t_u$. First, we let 
\begin{equation*}
    \rho_t^*(\delta, \alpha, \gamma) := 8 R \eta^2 \log \left(3 N / (\alpha \delta) \right) + 2 \alpha t M \left [ 8 B + \sqrt{8 \eta^2 \log(4MNt^2/\delta)} \right]  + 4 \gamma G^2.
\end{equation*}

\begin{lemma} For any $\delta > 0$, $\alpha > 0$ and $\gamma > 0$, if
\begin{equation}
    \mathcal{C}^t_{u} = \{ \vect{\theta} \in \widetilde{\mathcal{F}} : \|\vect{\theta} - \widehat{\vect{\theta}}^t_u \|_{E^{t}_{2,u}} \leq \sqrt{ \rho_t^*(\delta, \alpha, \gamma)}\}
\end{equation}
for all $t \in \mathbb{N}$, then
\begin{equation}
    \mathds{P} \left( \vect{\theta}_u^* \in \mathcal{C}^t_{u} \quad ,\forall t \in \mathbb{N} \right) \geq 1 - 2 \delta / N
\end{equation}
\label{lemma_confidence_contextual_user}

\end{lemma}

\begin{proof}

We start by computing the covering number $\mathcal{N}(\widetilde{\mathcal{F}}, \alpha, \| \cdot \|_{\text{2}})$. Let $\mathcal{B}^{\alpha}$ be an $\alpha$ covering of the ball $\mathcal{B} := \{ \vect{f} : \|\vect{f}\|_2 \leq 1\}$ in $L_2$-norm. Then, for any $\vect{\theta} \in \widetilde{\mathcal{F}}$, there exists $\vect{f} \in \mathcal{B}$ such that $\vect{\theta} = \vect{\Phi}^T \vect{f}$ and $\vect{f}_{\alpha} \in \mathcal{B}^{\alpha}$ such that $\|\vect{f}_{\alpha} - \vect{f}\|_2 \leq \alpha$. Hence, for $\vect{\theta}_\alpha = \vect{\Phi}^T \vect{f}_\alpha$, we have $\|\vect{\theta} - \vect{\theta}_\alpha\|_2 = \|\vect{\Phi}^T (\vect{f} - \vect{f}_\alpha)\|_2 \leq \alpha$. Therefore, the set $\widetilde{\mathcal{F}}^{\alpha} := \{ \vect{\theta} : \vect{\theta} = \vect{\Phi}^T \vect{f}, \vect{f} \in  \mathcal{B}^{\alpha}\}$ is an $\alpha$ cover of $\widetilde{\mathcal{F}}$ in $L_2$-norm. Taking the minimum over the size of covers, we obtain $\mathcal{N}(\widetilde{\mathcal{F}}, \alpha, \| \cdot \|_{\text{2}}) \leq (3/\alpha)^R$. Then, the result follows from Lemma \ref{lemma_confidence} by setting $N = 1$ (since we are only considering a single row) and writing $\mathcal{N}(\widetilde{\mathcal{F}}, \alpha, \| \cdot \|_{\text{2}}) \leq (3/\alpha)^R$.

\end{proof}

Now, in order to combine the guarantees for confidence intervals of different users in a single expression, let us define empirical $L_{2, \infty}$-norm as 

\begin{equation*}
    \|\vect{\Delta}\|_{E^t_{2, \infty}} := \max_{u \in \mathcal{N}} \| [\vect{\Delta}]_{u,:} \|_{E^{t}_{2, u}}
\end{equation*}

In other words, $\|\vect{\Delta}\|_{E^t_{2, \infty}}$ is equal to the maximum of emprical $L_2$-norms for the rows of $\vect{\Delta}$.

Then, we have the following guarantee

\begin{lemma} For any $\delta > 0$, $\alpha > 0$ and $\gamma > 0$, if
\begin{equation}
    \mathcal{C}_t = \{ \vect{\Theta} \in \mathcal{F} : \|\vect{\Theta} - \widehat{\vect{\Theta}}_t \|_{E^t_{2, \infty}} \leq \sqrt{ \rho_t^*(\delta, \alpha, \gamma)}\}
\end{equation}
for all $t \in \mathbb{N}$, then
\begin{equation}
    \mathds{P} \left( \vect{\Theta}^* \in \mathcal{C}_t \quad ,\forall t \in \mathbb{N} \right) \geq 1 - 2\delta
\end{equation}
\label{lemma_confidence_contextual}
\end{lemma}

\begin{proof}
We recall the definition of $\widehat{\vect{\Theta}}_t$:
\begin{equation}
    \widehat{\vect{\Theta}}_t = \argmin_{\vect{\Theta} \in \mathcal{F}} \left \{ L_{2,t}(\vect{\Theta}) + \gamma \|\vect{\Theta} - \overline{\vect{\Theta}} \|_\mathrm{F}^2 \right \}.
\end{equation}
Furthermore, note that $L_{2,t}(\vect{\Theta}) = \sum_{u \in \mathcal{N}} L^{t}_{2,u}([\vect{\Theta}]_{u,:})$ and $\|\vect{\Theta} - \overline{\vect{\Theta}} \|_\mathrm{F}^2 =  \sum_{u \in \mathcal{N}} \|[\vect{\Theta}]_{u, :} - \vect{\theta}^{\circ}_{u} \|_2^2$. Therefore, the rows of $\widehat{\vect{\Theta}}_t$ are equal to $\widehat{\vect{\theta}}^t_u$. Hence, by a union bound applied to Lemma \ref{lemma_confidence_contextual_user}, the result follows.
\end{proof}

Now, we continue with showing an analog of Lemma \ref{lemma_conf_width} for confidence regions $\mathcal{C}_t$

\begin{lemma}
Let $\vect{X}_t \in \mathcal{X}_t$ and $\vect{\Theta} \in \mathcal{C}_t$. Then,
\begin{equation}
    |\langle \vect{\Theta} - \widehat{\vect{\Theta}}^t, \vect{X}_t \rangle| \leq w_t \sqrt{|\mathcal{N}_t| \rho_t}
\end{equation}
where $w_t = \| \vect{X}_t \|_{\vect{A}_t^{-1}}$ is the "confidence width" of an action $\vect{X_t}$ at time $t$.
\label{conf_width_context}
\end{lemma}

\begin{proof} Let $\vect{\Delta} = \vect{\Theta} - \widehat{\vect{\Theta}}^t$. Then,
\begin{align*}
    |\langle \vect{\Delta} , \vect{X}_t \rangle| &= |\sum_{u \in \mathcal{N}_t} \sum_{i : (u,i) \in \mathcal{A}_t} \Delta_{ui}| \\
    &\leq \sum_{u \in \mathcal{N}_t} \sum_{i : (u,i) \in \mathcal{A}_t} |\Delta_{ui}| \\
    &\leq \sum_{u \in \mathcal{N}_t} \left ( \sum_{i : (u,i) \in \mathcal{A}_t} (n_{ui}^t + \gamma) \Delta_{ui}^2 \right )^{1/2}\left ( \sum_{i : (u,i) \in \mathcal{A}_t} \frac{1}{(n_{ui}^t + \gamma)} \right )^{1/2} \\
    &\leq \sum_{u \in \mathcal{N}_t} \sqrt{\rho_t} \left (\sum_{i : (u,i) \in \mathcal{A}_t} \frac{1}{(n_{ui}^t + \gamma)} \right )^{1/2} \\
    &\leq \sqrt{\rho_t} \sum_{u \in \mathcal{N}_t} \left (\sum_{i : (u,i) \in \mathcal{A}_t} \frac{1}{(n_{ui}^t + \gamma)} \right )^{1/2} \\
    &\leq \sqrt{\rho_t} \left ( |\mathcal{N}_t| \sum_{u \in \mathcal{N}_t}  \sum_{i : (u,i) \in \mathcal{A}_t} \frac{1}{(n_{ui}^t + \gamma)} \right )^{1/2} \\
    &\leq  w_t \sqrt{|\mathcal{N}_t| \rho_t} \\
\end{align*}

\end{proof}

\begin{lemma}
Define the regret at time $t$ as $r_t := \langle \vect{X}_t^*, \vect{\Theta}^* \rangle - \langle \vect{X}_t, \vect{\Theta}^* \rangle$. If $\vect{\Theta}^* \in \mathcal{C}_t$, then
\begin{equation*}
    r_t \leq \langle \vect{X}_t, \vect{\Theta}_t - \vect{\Theta}^* \rangle \leq 2 w_t \sqrt{|\mathcal{N}_t| \rho_t}
\end{equation*}
\label{lemma_regret_ub_contextual}
\end{lemma}

\begin{proof}
    The proof is similar to the proof of Lemma \ref{lemma_regret_ub} using the results of previous lemma.
\end{proof}

\begin{theorem}
If $\gamma \geq 1$ and $\vect{\Theta}^* \in \mathcal{C}_t$ for all $t \leq T$, then the T period social welfare regret is bounded by
\begin{equation*}
    \mathcal{R}^{SW}(T, \pi) \leq \sqrt{ 8 n N M \rho_T^* (\delta, \alpha, \gamma) T \log \left(1 + \frac{T}{\gamma} \right) }
\end{equation*}
\label{thm_regret_contextual}
\end{theorem}

\begin{proof}
Recall the definition of instantaneous regret at time $t$: $r_t = \langle \vect{X}_t^*, \vect{\Theta}^* \rangle - \langle \vect{X}_t, \vect{\Theta}^* \rangle$. Assuming that $\vect{\Theta}^* \in \mathcal{C}_t$ for all $t \leq T$, we have
\begin{align*}
    \mathcal{R}^{SW}(T, \pi) &= \sum_{t = 1}^{T} r_t \\
    &\leq \left( T \sum_{t = 1}^{T} r_t^2 \right)^{1/2}\\
    &\leq \left( 4 n \rho_T T \sum_{t = 1}^{T} w_t^2 \right)^{1/2}
\end{align*}
where the last step follows from Lemma \ref{thm_regret_contextual}. Then, using Lemma \ref{lemma_width_sum_ub}, we prove the result.

\end{proof}

\begin{theorem}
If the offered prices are $\vect{p}_t = \bar{\vect{p}}_t$, and $\vect{\Theta}^* \in \mathcal{C}_t$ for all $t \leq T$, then the T period instability is bounded as
    \begin{equation}
        \mathcal{R}^{I}(T, \pi)  \leq \sqrt{ 8 n NM \rho_T^*(\delta, \alpha, \gamma) T \log \left(1 + \frac{T}{\gamma} \right) }
    \end{equation}
\end{theorem}

\begin{proof}
Using Lemma \ref{instability_to_r} and \ref{lemma_regret_ub_contextual},
\begin{align*}
     \mathcal{R}^{I}(T, \pi) &\leq \sum_{t=1}^{T} \langle \vect{X}_t, \vect{\Theta}_{t} - \vect{\Theta}^{*} \rangle \\
    &\leq 2 \sqrt{n \rho_T T \sum_{t=1}^{T} w_t^2 } \\
    &\leq \sqrt{ 8 n \rho_T NM T \log \left(1 + \frac{T}{\gamma} \right) }
\end{align*}
\end{proof}

\begin{lemma}
    For any $\nu_t \geq 0$, if $\vect{\Theta}^* \in \mathcal{C}_t$, then
    \begin{equation*}
        \sum_{(u,i) \in \mathcal{A}_t} \mathbb{1} \left \{ \Delta^t_{ui} \geq \nu_t \right \} \leq \frac{2 w_t \sqrt{n \rho_t}}{\nu_t}
    \end{equation*}
\end{lemma}

\begin{proof}
     Using Lemma \ref{conf_width_context},
    \begin{align*}
        \sum_{(u,i) \in \mathcal{A}_t} \mathbb{1} \left \{ \Delta^t_{ui} \geq \nu_t \right \}
        &\leq \frac{1}{\nu_t} \sum_{(u,i) \in \mathcal{A}_t} \left| \Delta^t_{ui} \right| \\
        &\leq \frac{2 w_t \sqrt{n\rho_t}}{\nu_t}
    \end{align*}

\end{proof}

\subsection{Extension to the Model with Acceptance/Rejection Options}

\begin{theorem}    
In the model with acceptance/rejection options, if the offered prices $\vect{p}_t$ satisfy $\vect{p}_t = \bar{\vect{p}}_t - \nu_t$ where $\nu_t = \nu \sqrt{w_t}$ for some $\nu \geq 0$, and $\vect{\Theta}^* \in \mathcal{C}_t$ for all $t \leq T$, then the $T$ period social welfare regret is bounded as
\begin{equation*}
    \mathcal{R}^{SW}(T, \pi) \leq \left( 8 NM n \rho_T T \log \left(1 + \frac{T}{\gamma} \right) \right)^{1/2} + \frac{2}{\nu}  \left( 2 NM T^3 n^2 \rho_T^2 \log \left(1 + \frac{T}{\gamma} \right) \right)^{1/4}
\end{equation*}
\end{theorem}

\begin{proof}

Recall the definition of instantaneous regret at time $t$ in the model with acceptance/rejection options:
\begin{equation*}
    \widetilde{r}_t := \sum_{(u,i) \in \mathcal{A}^*_t} \Theta^*_{ui} - \sum_{(u,i) \in \mathcal{A}_t} \Theta^*_{ui} \mathds{1}\{ \Theta^*_{ui} \geq p_i^t \}
\end{equation*}
Recalling the definition of $r_t$, similar to proof of Lemma \ref{thm_regret_ar}, we can write
\begin{align*}
    \widetilde{r}_t \leq r_t + \sum_{(u,i) \in \mathcal{A}_t} \mathds{1}\{ \Delta^t_{ui} > \nu_t \}
\end{align*}
Then, we have 
\begin{align*}
    \mathcal{R}^{SW}(T, \pi) &= \sum_{t = 1}^{T} \widetilde{r}_t \\
    &\leq \sum_{t = 1}^{T} r_t + \sum_{t = 1}^{T} \sum_{(u,i) \in \mathcal{A}_t} \mathds{1}\{ \Delta^t_{ui} > \nu_t \}\\
    &\leq \left( T \sum_{t = 1}^{T} r_t^2 \right)^{1/2} + 2 \sqrt{n \rho_T} \sum_{t = 1}^{T} \frac{w_t}{\nu_t}\\
    &\leq \left( T \sum_{t = 1}^{T} r_t^2 \right)^{1/2} + \frac{2}{\nu}  \left( T^3 n^2 \rho_T^2 \sum_{t = 1}^{T} w_t^2 \right)^{1/4}\\
    &\leq \left( 8 NM n \rho_T T \log \left(1 + \frac{T}{\gamma} \right) \right)^{1/2} + \frac{2}{\nu}  \left( 2 NM T^3 n^2 \rho_T^2 \log \left(1 + \frac{T}{\gamma} \right) \right)^{1/4}
\end{align*}
\end{proof}

\begin{theorem}    
In the model with acceptance/rejection options, if the offered prices $\vect{p}_t$ satisfy $\vect{p}_t = \bar{\vect{p}}_t - \nu_t$ where $\nu_t = \nu \sqrt{w_t}$ for some $\nu \geq 0$, and $\vect{\Theta}^* \in \mathcal{C}_t$ for all $t \leq T$, then the $T$ period instability is bounded as
\begin{equation*}
    \mathcal{R}^{I}(T, \pi) \leq \left(8 n \rho_T NM \log \left(1 + \frac{T}{\gamma} \right) \right)^{1/2} + \nu n M \left( 2 NM T^3 \log \left(1 + \frac{T}{\gamma} \right) \right)^{1/4}
\end{equation*}
\end{theorem}

\begin{proof}
Using Lemma \ref{instability_to_r_accept_reject} and \ref{lemma_regret_ub_contextual},
\begin{align*}
     \mathcal{R}^{I}(T, \pi) &\leq \sum_{t=1}^{T} \langle \vect{X}_t, \vect{\Theta}_{t} - \vect{\Theta}^{*} \rangle +  \nu n M \sum_{t=1}^{T} \sqrt{w_t} \\
    &\leq 2 \sqrt{n \rho_T} \sum_{t=1}^{T} w_t +  \nu n M \sum_{t=1}^{T} \sqrt{w_t}\\
    &\leq \left(8 n \rho_T NM \log \left(1 + \frac{T}{\gamma} \right) \right)^{1/2} + \nu n M \left( 2 NM T^3 \log \left(1 + \frac{T}{\gamma} \right) \right)^{1/4}
\end{align*}

\end{proof}

\begin{theorem}  
In the model with acceptance/rejection options, if the offered prices $\vect{p}_t$ satisfy $\vect{p}_t = \bar{\vect{p}}_t - \nu_t$ where $\nu_t = \nu \sqrt{w_t}$ for $\nu = \left( \frac{4 \rho_T}{n M^2} \right)^{1/4}$, then, with probability $1 - 2 \delta$, the algorithm results in regrets that satisfy
\begin{align*}
    \mathcal{R}^{SW}(T, \pi) &\leq \sqrt{\kappa_T nM T} + (\kappa_T)^{\frac{1}{4}} (nMT)^{\frac{3}{4}} \\
    \mathcal{R}^{I}(T, \pi) &\leq \sqrt{\kappa_T nM T} + (\kappa_T)^{\frac{1}{4}} (nMT)^{\frac{3}{4}}
\end{align*}
where $\kappa_T = 8 N \rho_T \log (1 + T/\gamma)$.
\label{thm_regret_and_instability_ar_context}
\end{theorem}

\begin{proof}
We have
\begin{align*}
    \mathcal{R}^{SW}(T, \pi) &\leq \left( 8 NM n \rho_T T \log \left(1 + \frac{T}{\gamma} \right) \right)^{1/2} + \frac{2}{\nu}  \left( 2 NM T^3 n^2 \rho_T^2 \log \left(1 + \frac{T}{\gamma} \right) \right)^{1/4} \\
    &=  (\kappa_T nM T)^{\frac{1}{2}} + (\kappa_T)^{\frac{1}{4}} (nMT)^{\frac{3}{4}}
\end{align*}

Similarly, we have
\begin{align*}
    \mathcal{R}^{I}(T, \pi) &\leq \left(8 n \rho_T NM \log \left(1 + \frac{T}{\gamma} \right) \right)^{1/2} + \nu n M \left( 2 NM T^3 \log \left(1 + \frac{T}{\gamma} \right) \right)^{1/4}\\
    &= (\kappa_T nM T)^{\frac{1}{2}} + (\kappa_T)^{\frac{1}{4}} (nMT)^{\frac{3}{4}}
\end{align*}

\end{proof}

\begin{theorem}    
In the model with acceptance/rejection options, Letting $\delta = \mathcal{O} ((NMT)^{-1})$, $\alpha = \mathcal{O} ((NMT)^{-1})$, $\gamma = 1$ and $\nu = \mathcal{O} \left(\left( \frac{\eta^2 n R}{N^2 M^2}\right)^{1/4}\right)$ results in regret bounds that satisfy
\begin{align*}
    \mathcal{R}^{SW}(T, \pi) = \widetilde{\mathcal{O}} \left( \left( \eta^2 n R \right)^{\frac{1}{4}} (NMT)^{\frac{3}{4}}  \right) \\
    \mathcal{R}^{I}(T, \pi) = \widetilde{\mathcal{O}} \left( \left( \eta^2 n R \right)^{\frac{1}{4}} (NMT)^{\frac{3}{4}}  \right)
\end{align*}
\end{theorem}

\begin{proof}
By Theorem \ref{thm_regret_and_instability_ar_context}, with probability $1$,
\begin{equation*}
    \mathcal{R}^{SW}(T, \pi) \leq (1 - 2 \delta) \left((\kappa_T NM T)^{\frac{1}{2}} + (\kappa_T)^{\frac{1}{4}} (NMT)^{\frac{3}{4}} \right)  + 2 \delta NMT
\end{equation*}

Similarly, with probability $1$, we have
\begin{equation*}
    \mathcal{R}^{I}(T, \pi) \leq (1 - 2 \delta)  \left ( (\kappa_T NM T)^{\frac{1}{2}} + (\kappa_T)^{\frac{1}{4}} (NMT)^{\frac{3}{4}} \right) + 2 \delta NMT
\end{equation*}

Letting $\delta = \mathcal{O} ((NMT)^{-1})$, $\alpha = \mathcal{O} ((NMT)^{-1})$, $\gamma = 1$, we conclude the proof.

\end{proof}

\section{MARTINGALE EXPONENTIAL INEQUALITIES}

Consider a sequence of random variables $(Z_n)_{n \in \mathds{N}}$ adapted to the filtration $(\mathcal{H}_n)_{n \in \mathds{N}}$. Assume $\mathds{E}[\exp (\lambda Z_i)]$ is finite for all $\lambda$. Define the conditional mean $\mu_i = \mathds{E}[Z_i | \mathcal{H}_{i-1}]$, and define the conditional cumulant generating function of the centered random variable $[Z_i - \mu_i]$ by $\psi_i(\lambda) := \log \mathds{E}[ \exp (\lambda [Z_i - \mu_i]) | \mathcal{H}_{i-1}]$. Let 
\begin{equation*}
    M_n(\lambda) = \exp \left \{ \sum_{i=1}^{n} \lambda [Z_i - \mu_i] - \psi_i(\lambda) \right \} 
\end{equation*}

\begin{lemma}
$(M_n (\lambda))_{n \in \mathds{N}}$ is a martingale with respect to the filtration $(\mathcal{H}_n)_{n \in \mathds{N}}$, and $\mathds{E}[M_n (\lambda)] = 1$.
\label{martingale}
\end{lemma}

\begin{proof}
By definition, we have
\begin{equation*}
    \mathds{E}[M_1 (\lambda) | \mathcal{H}_0] = \mathds{E}[\exp \{ \lambda [Z_1 - \mu_1] - \psi_1(\lambda) \} | \mathcal{H}_0] = 1 
\end{equation*}
Then, for any $n \geq 2$,
\begin{align*}
    \mathds{E}[M_n (\lambda) | \mathcal{H}_{n-1}] &= \mathds{E}[M_{n-1}(\lambda) \exp \{ \lambda [Z_n - \mu_n] - \psi_n(\lambda) \} | \mathcal{H}_{n-1}] \\
    &= M_{n-1}(\lambda) \mathds{E}[\exp \{ \lambda [Z_n - \mu_n] - \psi_n(\lambda) \} | \mathcal{H}_{n-1}]\\
    &= M_{n-1}(\lambda)\\
\end{align*}
since $M_{n-1}(\lambda)$ is a measurable function of the filtration $\mathcal{H}_{n-1}$.
\end{proof}

\begin{lemma}
For all $x \geq 0$ and $\lambda \geq 0$, 
\begin{equation*}
\mathds{P} \left(\sum_{i=1}^{n} \lambda Z_i \leq x + \sum_{i=1}^{n} [\lambda \mu_i + \psi_i(\lambda)] \quad ,\forall t \in \mathds{N} \right) \geq 1 - e^{-x}
\end{equation*}
\label{exp_martingale}
\end{lemma}

\begin{proof}
For any $\lambda$, $(M_n (\lambda))_{n \in \mathds{N}}$ is a martingale with respect to $(\mathcal{H}_n)_{n \in \mathds{N}}$ and $\mathds{E}[M_n (\lambda)] = 1$ by Lemma \ref{martingale}. For arbitrary $x \geq 0$, define $\tau_x = \inf \{ n\geq 0 | M_n(\lambda) \geq x \}$ and note that $\tau_x$ is a stopping time corresponding to the first time $M_n$ crosses the boundary $x$. Since $\tau$ is a stopping time with respect to $(\mathcal{H}_n)_{n \in \mathds{N}}$, we have $\mathds{E}[M_{\tau_x \wedge n}(\lambda)] = 1$. Then, by Markov's inequality
\begin{equation*}
    x \mathds{P} (M_{\tau_x \wedge n}(\lambda) \geq x) \leq \mathds{E}[M_{\tau_x \wedge n}(\lambda)] = 1
\end{equation*}

Noting that the event $\{ M_{\tau_x \wedge n}(\lambda) \geq x \} = \bigcup_{k=1}^n \{M_{k}(\lambda) \geq x\} $, we have
\begin{equation*}
\mathds{P} \left( \bigcup_{k=1}^n \{M_{k}(\lambda) \geq x\} \right) \leq \frac{1}{x}
\end{equation*}

Taking the limit as $n \to \infty$, and applying monotone convergence theorem shows that $\mathds{P} \left( \bigcup_{k=1}^\infty \{M_{k}(\lambda) \geq x\} \right) \leq \frac{1}{x}$ or $\mathds{P} \left( \bigcup_{k=1}^\infty \{M_{k}(\lambda) \geq e^x\} \right) \leq e^{-x}$. Then, by definition of $M_k (\lambda)$, we conclude
\begin{equation*}
    \mathds{P} \left( \bigcup_{k=1}^\infty \left \{ \sum_{i=1}^{n} \lambda [Z_i - \mu_i] - \psi_i(\lambda) \geq x \right \}  \right) \leq e^{-x}
\end{equation*}

\end{proof}

\section{ADDITIONAL EXPERIMENTAL RESULTS AND DETAILS}
\label{sect_additional_exp}

All experiments are implemented in Python and carried out on a cluster with Lenovo NeXtScale nx360m5 nodes equipped with two Intel Xeon 12-core Haswell processors (24 cores per node) with core frequency 2.3 GHz and 64 GB of memory. We solve the allocation program \eqref{integer_num} using large-scale mixed integer programming (MIP) solver packages to have efficient computations.

\textbf{Parameter setup}:
\begin{itemize}[labelindent= 0pt, align= left, labelsep=0.4em, leftmargin=*]
\item In all data sets, $\vect{\Theta}^*$ is scaled such that $\|\vect{\Theta}^*\|_{\infty} \leq 1$.
\item Standard deviation of the rewards: $\eta = 0.2$.
\item In the static setting, $d^{t}_u = 1$ for all $u \in [N]$.
\item In the dynamic setting, $d^{t}_u = 1$ with probability $0.2$, $0$ otherwise, independently for each $u \in [N]$.
\item $C_\text{max} = \text{ceil} \left( \frac{1}{M} \sum_{u = 1}^{N} d^{t}_u \right)$.
\item $c_{t, i}$ are uniformly sampled over $\{1, \dots, C_\text{max}\}$ independently for each $t \in [T]$ and $i \in [M]$.

\end{itemize}



In Figures \ref{exp_1}, \ref{exp_2}, \ref{exp_3} and \ref{exp_4}, we provide detailed  results for different experimental settings described in Section \ref{sect_exp}.

\begin{figure}[ht]
\center
\includegraphics{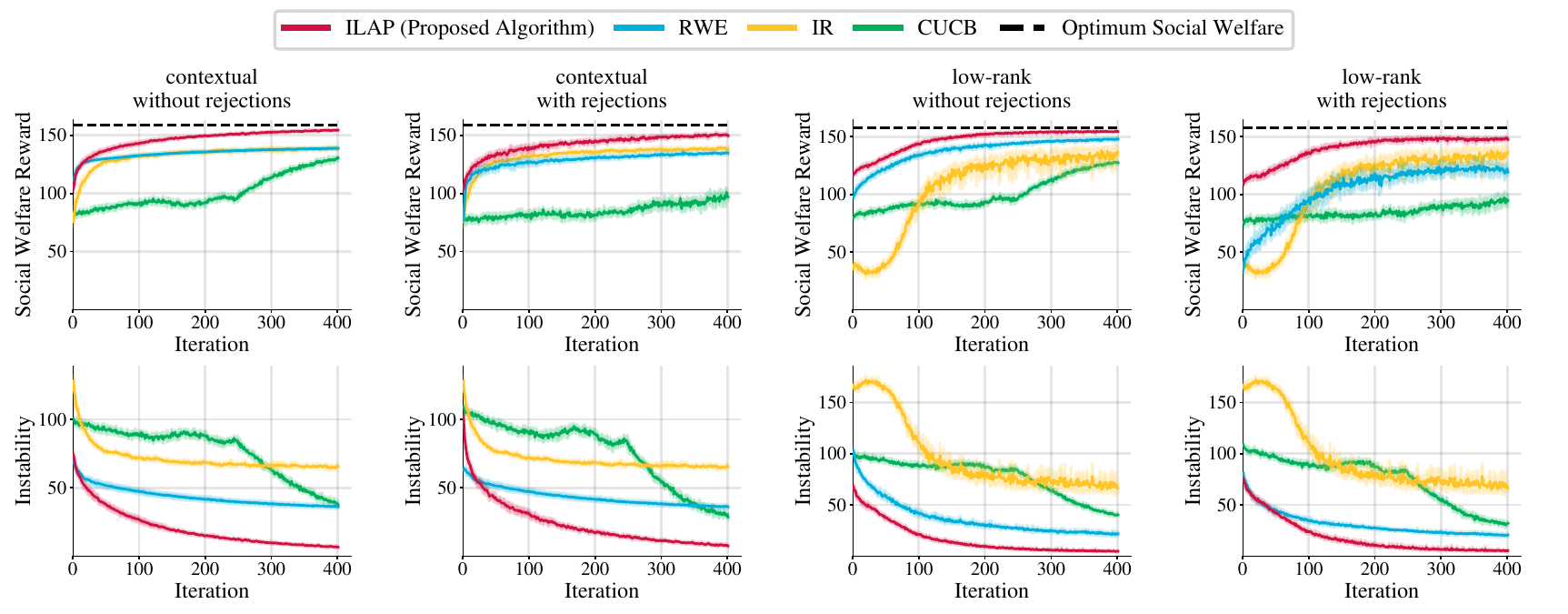}
\caption{Experimental results for synthetic data in a static setting with $N = 250$, $M = 200$, $R = 20$. The experiments are run on $20$ problem instances and means are reported together with error regions that indicate one standard deviation of uncertainty.}
\label{exp_1}
\end{figure}

\begin{figure}[ht]
\center
\includegraphics{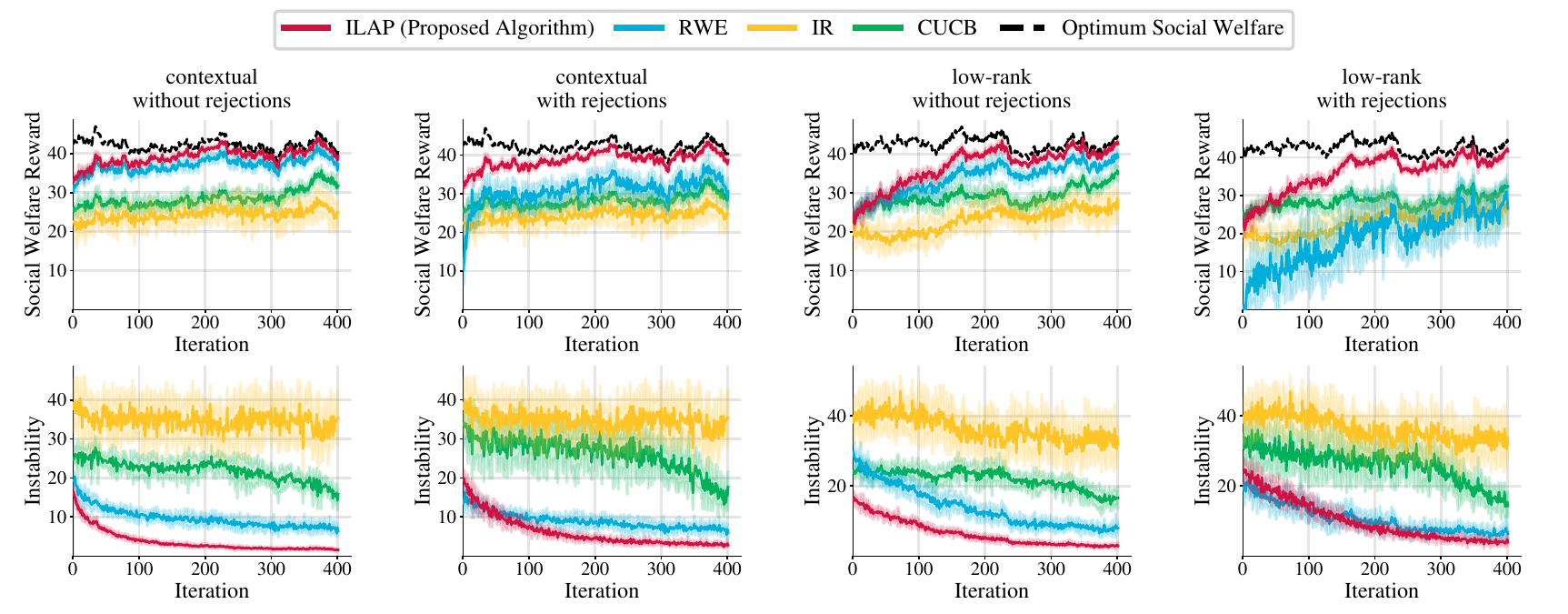}
\caption{Experimental results for synthetic data in a dynamic setting with $N = 350$, $M = 50$, $R = 10$, probability of activity $0.2$. The experiments are run on $20$ problem instances and means are reported together with error regions that indicate one standard deviation of uncertainty.}
\label{exp_2}
\end{figure}

\begin{figure}[ht]
\center
\includegraphics{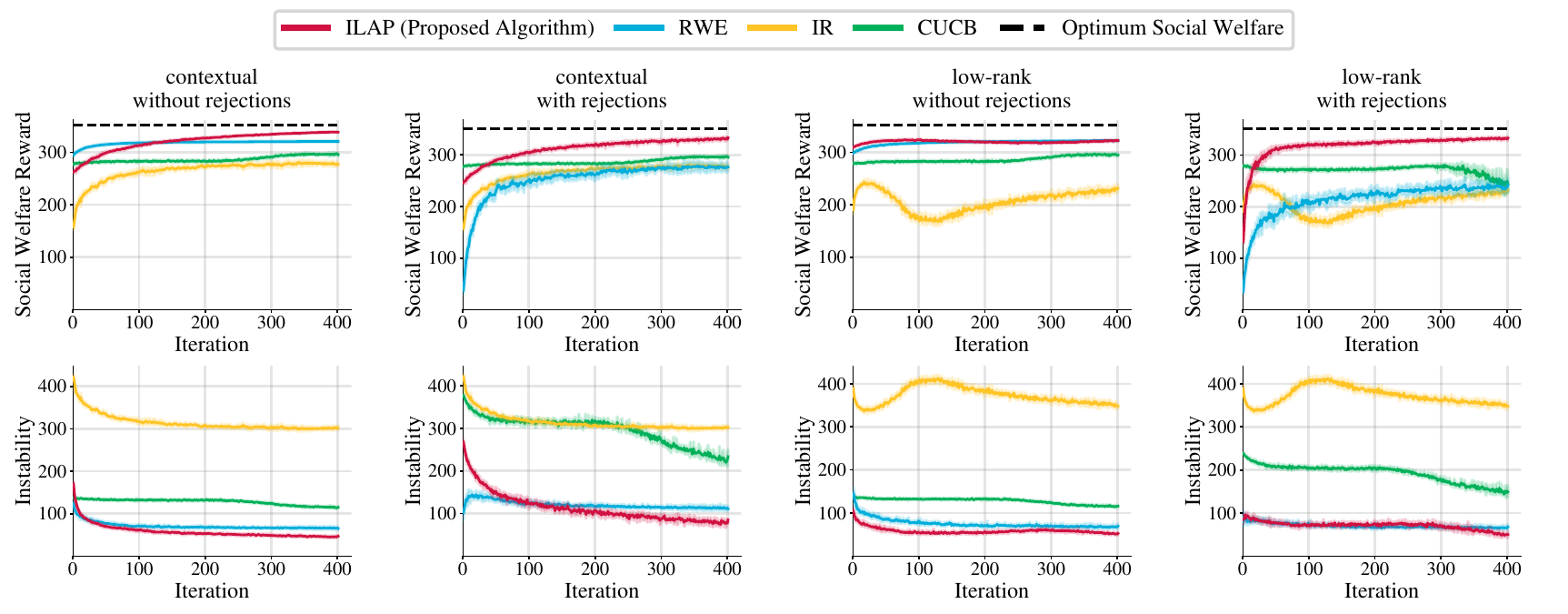}
\caption{Experimental results for MovieLens 100k data in a static setting. The experiments are run on $20$ problem instances and means are reported together with error regions that indicate one standard deviation of uncertainty.}
\label{exp_3}
\end{figure}

\begin{figure}[ht]
\center
\includegraphics{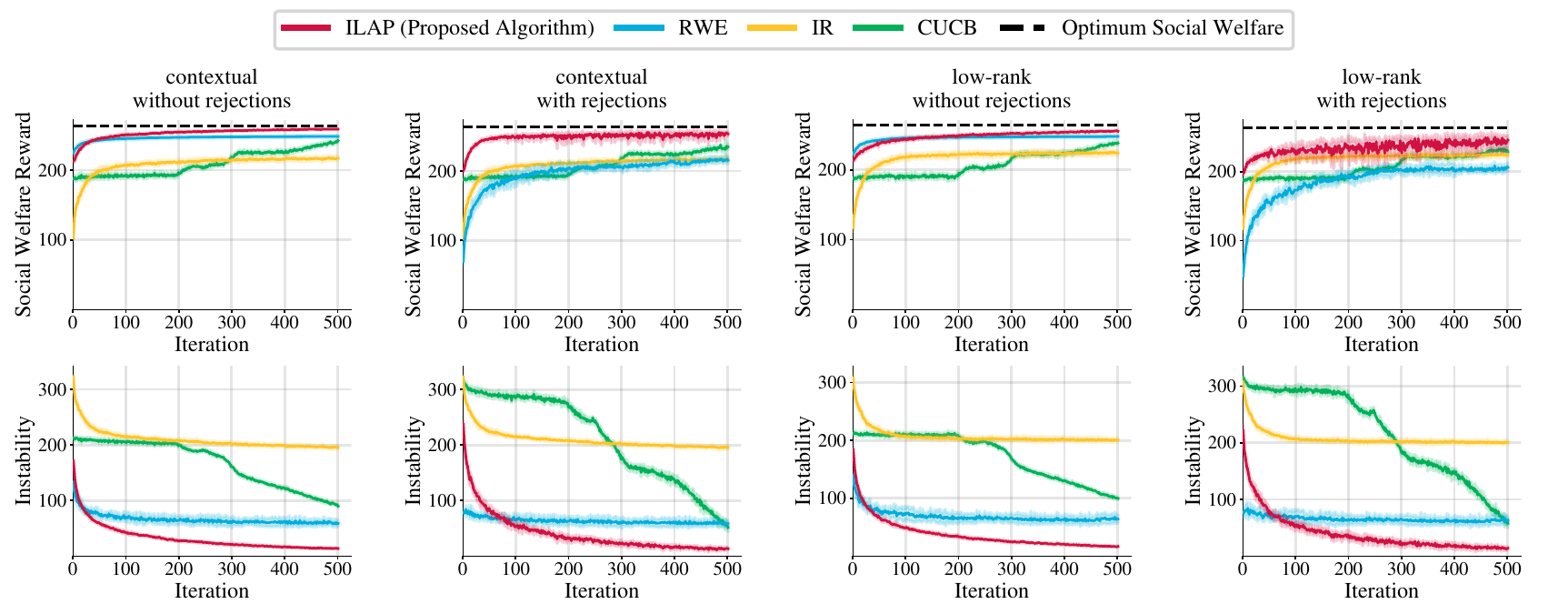}
\caption{Experimental results for Yelp data in a static setting. The experiments are run on $20$ problem instances and means are reported together with error regions that indicate one standard deviation of uncertainty.}
\label{exp_4}
\end{figure}

\end{document}